\RequirePackage[l2tabu,orthodox]{nag}		% for deprecated commands
\documentclass[reqno]{amsart}		% for amsart
\usepackage[margin=1.5in,bottom=1.25in]{geometry}		% for tighter margins (80 CPL)

\usepackage{amsmath}		% for AMS macros
\usepackage{amssymb}		% for AMS symbols
\usepackage{amsfonts}		% for AMS fonts
\usepackage{amsthm}		% for theorems
\usepackage[foot]{amsaddr}		% for author footnotes

\usepackage{mathtools}		% for advanced math

\mathtoolsset{%
}

\newtagform{comment}{\{}{\}}		% for equation comments

\usepackage[utf8]{inputenc}		% for source encoding
\usepackage[T1]{fontenc}		% for font encoding

\usepackage[%		% for math font selection
cal=cm,
]
{mathalfa}

\usepackage[proportional,tabular,lining,sf,mono=false]{libertine}
\usepackage{dsfont}		% for blackboard bold font
\usepackage{inconsolata}

\usepackage{acronym}		% for acronyms
\newcommand{\acdef}[1]{\textit{\acl{#1}} \textup{(\acs{#1})}\acused{#1}}		% for acro def
\usepackage[labelfont={bf,small},labelsep=colon,font=small]{caption}	% for caption control
\usepackage[dvipsnames,svgnames]{xcolor}		% for color
\colorlet{MyRed}{Crimson!75!Black}
\colorlet{MyBlue}{MediumBlue}
\colorlet{MyGreen}{DarkGreen!80!Black}

\newcommand{\afterhead}{.}
\newcommand{\ackperiod}{}		% for period bug in acknowledgments
\newcommand{\para}[1]{\medskip\paragraph{\textbf{#1\afterhead}}}

\usepackage{pifont}

\usepackage{wasysym}		% for extra symbols

\usepackage{subcaption}		% for subfigures
\usepackage{tikz}		% for figures
\usetikzlibrary{calc,patterns}

\usepackage{array}		% for flexible tables
\usepackage{booktabs}		% for better tables
\usepackage[inline,shortlabels]{enumitem}		% for list control
\setenumerate{itemsep=2pt,topsep=2pt}

\usepackage[kerning=true]{microtype}		% for microtypography

\usepackage{latexsym}		% for symbols
\usepackage{xspace}		% for flexible spaces}

\usepackage[numbers,sort&compress]{natbib}		% for citations

\newcommand{\citef}[2][]{\citeauthor{#2} \cite[#1]{#2}}
\usepackage{hyperref}
\hypersetup{
colorlinks=true,
linktocpage=true,
pdfstartview=FitH,
breaklinks=true,
pdfpagemode=UseNone,
pageanchor=true,
pdfpagemode=UseOutlines,
plainpages=false,
bookmarksnumbered,
bookmarksopen=false,
bookmarksopenlevel=1,
hypertexnames=true,
pdfhighlight=/O,
urlcolor=MyRed,linkcolor=MyBlue,citecolor=MyBlue,	% for on-screen
pdftitle={},
pdfauthor={},
pdfsubject={},
pdfkeywords={},
pdfcreator={pdfLaTeX},
pdfproducer={LaTeX with hyperref}
}

\usepackage[sort&compress,capitalize,nameinlink]{cleveref}		% for cleveref formatting
\crefname{algorithm}{Alg.}{Algs.}
\usepackage{algorithm}		% for algorithm environments
\usepackage{algpseudocode}		% for algorithm macros
\usepackage{thmtools}		% for theorem tools
\usepackage{thm-restate}		% for restating theorems

\theoremstyle{plain}
\newtheorem{corollary}{Corollary}		% for corollaries
\newtheorem{lemma}{Lemma}		% for lemmas
\newtheorem{proposition}{Proposition}		% for propositions

\newtheorem*{corollary*}{Corollary}		% for corollaries (unnumbered)

\theoremstyle{definition}
\newtheorem{definition}{Definition}		% for definitions
\newtheorem{assumption}{Assumption}		% for assumptions
\newtheorem{example}{{\raisebox{1pt}{\small$\blacktriangleright$}} Example}		% for examples

\newtheorem*{definition*}{Definition}		% for definitions (unnumbered)
\newtheorem*{assumption*}{Assumptions}		% for assumptions (unnumbered)
\newtheorem*{example*}{Example}		% for examples (unnumbered)

\theoremstyle{remark}
\newtheorem{remark}{Remark}		% for remarks
\newtheorem*{remark*}{Remark}		% for remarks (unnumbered)

\def\endenv{\hfill\raisebox{1pt}{\small$\blacktriangleleft$}}
\newcounter{proofpart}

\usepackage[textwidth=30mm,disable]{todonotes}		% for todo's
\newcommand{\debug}[1]{#1}		% for removing macro coloring

\newcommand{\newmacro}[2]{\newcommand{#1}{\debug{#2}}}		% for shorthand definitions
\newcommand{\newop}[2]{\DeclareMathOperator{#1}{\debug{#2}}}		% for shorthand definitions

\DeclarePairedDelimiter{\bracks}{[}{]}		% for brackets
\DeclarePairedDelimiter{\parens}{(}{)}		% for parentheses

\DeclarePairedDelimiter{\abs}{\lvert}{\rvert}		% for absolute value
\DeclarePairedDelimiter{\ceil}{\lceil}{\rceil}		% for ceiling
\DeclarePairedDelimiterX{\inner}[2]{\langle}{\rangle}{#1,#2}		% for scalar product
\DeclarePairedDelimiter{\norm}{\lVert}{\rVert}		% for norm
\DeclarePairedDelimiterXPP{\dnorm}[1]{}{\lVert}{\rVert}{_{\ast}}{#1}		% for dual norm
\DeclarePairedDelimiterXPP{\tvnorm}[1]{}{\lVert}{\rVert}{_{\mathrm{TV}}}{#1}		% for dual norm
\DeclarePairedDelimiterXPP{\onenorm}[1]{}{\lVert}{\rVert}{_{1}}{#1}		% for dual norm
\DeclarePairedDelimiterXPP{\twonorm}[1]{}{\lVert}{\rVert}{_{2}}{#1}		% for dual norm
\DeclarePairedDelimiterXPP{\supnorm}[1]{}{\lVert}{\rVert}{_{\infty}}{#1}		% for dual norm
\DeclarePairedDelimiterX{\braket}[2]{\langle}{\rangle}{#1,#2}		% for brakets
\DeclarePairedDelimiterX{\setdef}[2]{\{}{\}}{#1:#2}		% for set builder notation
\DeclarePairedDelimiterX{\window}[2]{[}{]}{#1\,.\,.\,#2}		% for set builder notation
\DeclarePairedDelimiterXPP{\exclude}[1]{\mathopen{}\setminus}{\{}{\}}{}{#1}

\newcommand{\alt}[1]{#1'}		% for alternates

\newcommand{\R}{\mathbb{R}}		% for reals
\DeclareMathOperator*{\argmax}{arg\,max}		% for argmax
\DeclareMathOperator*{\argmin}{arg\,min}		% for argmin
\DeclareMathOperator{\bigoh}{\mathcal{O}}		% for Landau O
\DeclareMathOperator{\diam}{diam}		% for diameter
\DeclareMathOperator{\dif}{D}		% for differential
\DeclareMathOperator{\dist}{dist}		% for distance
\newop{\dom}{dom}		% for domain
\DeclareMathOperator{\one}{\mathds{1}}		% for indicator
\DeclareMathOperator{\supp}{supp}		% for support
\DeclareMathOperator{\unif}{unif}		% for uniform distribution
\newcommand{\subs}{\leftarrow}      % for substitution

\newmacro{\coef}{\alpha}		% for coefficient
\newmacro{\dd}{\:d}		% for integrators

\newcommand{\eps}{\varepsilon}		% for better epsilon
\newcommand{\insum}{\sum\nolimits}		% for compact sums
\newcommand{\cf}{cf.\xspace}		% for consistency
\newcommand{\eg}{e.g.,\xspace}		% for consistency
\newcommand{\ie}{i.e.,\xspace}		% for consistency
\newcommand{\vs}{vs.\xspace}		% for consistency

\newcommand{\textpar}[1]{\textup(#1\textup)}		% for upshape parentheses

\newcommand{\txs}{\textstyle}		% for forcing inline style

\newcommand{\from}{\colon}		% for function definition
\newcommand{\defeq}{\coloneqq}		% for direct definition
\newcommand{\eqdef}{\eqqcolon}		% for reverse definition

\newmacro{\set}{\mathcal{S}}		% for generic set

\newmacro{\points}{\mathcal{K}}		% for point set
\newmacro{\intpoints}{\points^{\circ}}		%for point set interior
\newmacro{\point}{x}		% for generic point
\newmacro{\pointalt}{\alt\point}		% for alternate point

\newmacro{\primal}{p}		% for primal point
\newmacro{\proxal}{q}		% for primal point
\newmacro{\primalt}{q}		% for alternate primal
\newmacro{\dpoints}{\mathcal{Y}}		% for second point set (duals, etc.)
\newmacro{\dpoint}{y}		% for second generic point
\newmacro{\dpointalt}{\alt\dpoint}		% for second alternate variable

\newmacro{\base}{p}		% for reference point
\newmacro{\basealt}{q}		% for alternate reference point

\newcommand{\bench}{\pi^{\ast}}		% for test strat
\newcommand{\benchdist}{\pdist^{\ast}}		% for test dist

\newmacro{\pointest}{\point}		% for test point (\point by default)
\newmacro{\stratest}{\strat}		% for test strat (\point by default)
\newmacro{\ptest}{\pdist}		% for test density (\point by default)
\newmacro{\open}{\mathcal{U}}		% for open sets
\newmacro{\closed}{\mathcal{C}}		% for closed sets
\newmacro{\cpt}{\mathcal{K}}		% for compact sets
\newmacro{\nhd}{\mathcal{U}}		% for neighborhoods

\newmacro{\start}{1}		% for start index
\newmacro{\runstart}{\tau}		% for running start
\newmacro{\running}{1,2,\dotsc}		% for running index
\newmacro{\run}{t}		% for main sequence index
\newmacro{\runalt}{s}		% for alternate sequence index
\newmacro{\nRuns}{T}		% for total number of runs
\newmacro{\runtime}{\theta}		% for runtime

\newmacro{\runs}{\mathcal{\nRuns}}		% for set of indices
\newmacro{\subruns}{\mathcal{S}}		% for set of indices

\newcommand{\new}[1]{#1^{+}}		% for new iterate

\newmacro{\choice}{\point}		% for chosen action
\newmacro{\policy}{\pi}		% for policy
\newmacro{\state}{p}		% for state of algorithm 
\newmacro{\score}{y}		% for score

\newmacro{\aux}{\tilde\state}		% for auxiliary state 
\newmacro{\daux}{\tilde\score}		% for dual auxiliary state

\newmacro{\step}{\gamma}		% for step-size
\newmacro{\temp}{\eta}		% for learning rate

\newmacro{\measures}{\mathcal{M}}		% for Borel space
\newmacro{\conts}{C(\points)}		% for Borel space
\newmacro{\linf}{\mathcal{L}^{\infty}(\points)}		% for Borel space
\newmacro{\squareints}{\mathcal{L}^{2}(\points)}		% for Borel space

\newmacro{\banach}{\mathcal{V}}		% for Banach space
\newmacro{\hilbert}{\mathcal{H}}		% for Banach space

\newmacro{\fun}{\varphi}		% for test function

\newmacro{\vecspace}{\mathcal{V}}		% for generic vector space
\newmacro{\vdim}{n}		% for dimension
\newmacro{\vvec}{x}		% for generic vector
\newmacro{\bvec}{e}		% for basis vectors
\newmacro{\unitvec}{u}		% for unit vectors

\newmacro{\subspace}{\mathcal{W}}		% for subspace
\newmacro{\wvec}{w}		% for generic subspace vector

\newmacro{\tanhull}{\mathcal{Z}}		% for tangent hull
\newmacro{\tanvec}{z}		% for tangent vectors

\newmacro{\dspace}{\vecspace^{\ast}}		% for dual space
\newmacro{\dual}{\psi}		% for primal point
\newmacro{\dvec}{v}		% for dual vector
\newmacro{\dbvec}{\eps}		% for dual basis vectors

\newmacro{\ones}{\mathbf{1}}		% for vector of ones
\newmacro{\mat}{M}		% for generic matrix
\newmacro{\eye}{I}		% for identity matrix

\newop{\tspace}{T}		% for tangent space
\newop{\tcone}{TC}		% for tangent cone
\newop{\dcone}{\tcone^{\ast}}		% for dual cone
\newop{\ncone}{NC}		% for normal cone
\newop{\pcone}{PC}		% for polar cone

\newmacro{\cvx}{\mathcal{C}}		% for generic convex set
\newmacro{\subd}{\partial}		% for subdifferential
\newmacro{\hmat}{H}		% for Hessian matrix

\newop{\Opt}{Opt}		% for value of problem
\newop{\Sol}{Sol}		% for solution of problem

\newmacro{\obj}{f}		% for objective function
\newmacro{\objalt}{g}		% for alternative objective function
\newmacro{\sobj}{F}		% for stochastic objective

\newmacro{\param}{\theta}		% for parameter
\newmacro{\params}{\Theta}		% for parameter space

\newmacro{\gvec}{g}		% for gradient vector
\newmacro{\vecfield}{v}		% for vector field

\newmacro{\gbound}{G}		% for gradient bound
\newmacro{\vbound}{R}		% for field bound

\newcommand{\sol}[1][\point]{#1^{\ast}}		% for solutions (x by default)
\newmacro{\strong}{\ell}		% for strong convexity modulus
\newmacro{\lips}{L}		% for Lipschitz modulus
\newmacro{\hold}{L}		% for Hölder modulus
\newmacro{\hexp}{\alpha}		% for Hölder exponent

\newmacro{\minmax}{\Phi}		% for minmax objective

\newmacro{\minvar}{x}		% for minimization variable
\newmacro{\minvaralt}{\alt x}		% for alternate minvar
\newmacro{\minvars}{\mathcal{X}}		% for minvar space

\newmacro{\maxvar}{y}		% for maximization variable
\newmacro{\maxvaralt}{\alt y}		% for alternate maxvar
\newmacro{\maxvars}{\mathcal{Y}}		% for maxvar space

\newop{\NE}{NE}		% for Nash equilibrium
\newop{\CE}{CE}		% for CE set
\newop{\CCE}{CCE}		% for Hannan set

\newop{\brep}{br}		% for best responses
\newop{\reg}{Reg}		% for regret
\newop{\regalt}{\widetilde{Reg}}		% for alternative regret - wide tilde does not have the best aesthetics 
\newop{\preg}{\overline{Reg}}		% for pseudo-regret
\newop{\dynreg}{DynReg}		% for dynamic regret
\newop{\val}{val}		% for value of game

\newmacro{\strat}{\pi}		% for mixed strategy
\newmacro{\stratalt}{\psi}		% for mixed strategy
\newmacro{\strats}{\simplex}		% for set of mixed strategies
\newmacro{\intstrats}{\strats^{\circ}}		% for set of interior strategies

\newmacro{\abscont}{\text{cont}}
\newmacro{\sing}{\perp}

\newmacro{\pdist}{p}		% for density
\newmacro{\dirac}{\delta}		% for Dirac
\newmacro{\simple}{\chi}		% for simple strategy
\newmacro{\simples}{\mathcal{X}}		% for set of densities

\newmacro{\play}{i}		% for main player index
\newmacro{\playalt}{j}		% for alternate player index
\newmacro{\nPlayers}{N}		% for number of players
\newmacro{\players}{\mathcal{\nPlayers}}		% for set of players

\newmacro{\pure}{a}		% for main strategy index
\newmacro{\purealt}{a'}		% for alternate strategy index
\newmacro{\nPures}{n}		% for number of strategies
\newmacro{\pures}{\mathcal{A}}		% for set of strategies

\newmacro{\cost}{c}		% for cost function
\newmacro{\loss}{\ell}		% for loss function
\newmacro{\pay}{u}		% for payoff function
\newmacro{\payv}{v}		% for payoff vector
\newmacro{\pot}{\obj}		% for potential function

\newmacro{\game}{\mathcal{G}}		% for game
\newmacro{\gamefull}{\game(\players,\points,\pay)}		% for full game

\newmacro{\fingame}{\Gamma}		% for finite game
\newmacro{\fingamefull}{\Gamma(\players,\pures,\pay)}		% for full finite game

\newop{\Eucl}{\Pi}		% for Euclidean projection
\newop{\logit}{\Lambda}		% for logit map

\newmacro{\hreg}{h}		% for regularizer
\newmacro{\hconj}{\hreg^{\ast}}		% for conjugate
\newmacro{\hdec}{\theta}		% for decomposable regularizer
\newmacro{\breg}{D}		% for Bregman divergence
\newmacro{\pmap}{P}		% for prox-mapping
\newmacro{\mirror}{Q}		% for mirror map
\newmacro{\fench}{F}		% for Fenchel coupling
\newmacro{\hstr}{K}		% for strong convexity constant
\newmacro{\depth}{H}		% for regularizer depth
\newmacro{\zone}{\mathbb{D}}		% for Bregman zone
\newmacro{\subpoints}{\points^{\circ}}		% for Bregman interior
\newmacro{\proxdom}{\mathcal{Q}}		% for prox-domain

\newmacro{\energy}{E}		% for energy function
\newmacro{\hvol}{\phi}		% for regularizer volume

\DeclareMathOperator{\ex}{\mathbb{E}}		% for expectations
\DeclareMathOperator{\prob}{\mathbb{P}}		% for probability
\DeclareMathOperator{\simplex}{\Delta}		% for simplices

\newmacro{\sample}{\omega}		% for samples
\newmacro{\samples}{\Omega}		% for sample space
\newmacro{\filter}{\mathcal{F}}		% for filtrations
\newmacro{\history}{\mathcal{F}}		% for history
\newmacro{\probspace}{(\samples,\filter,\prob)}		% for probability space

\newmacro{\meas}{\mu}		% for base measure
\newmacro{\leb}{\lambda}		% for base measure

\newmacro{\event}{E}       % for event
\newmacro{\eventalt}{H}       % for alternate event
\newmacro{\mean}{\mu}		% for mean of distribution
\newmacro{\sdev}{\sigma}		% for mean of distribution
\newmacro{\variance}{\sdev^{2}}		% for mean of distribution

\newmacro{\dkl}{D_{\mathrm{KL}}}		% for Kullback Leibler
\providecommand\given{}		% empty command for conditionals

\DeclarePairedDelimiterXPP{\exof}[1]{\ex}{[}{]}{}{%		% for conditional expectations
\renewcommand\given{\nonscript\,\delimsize\vert\nonscript\,\mathopen{}} #1}

\DeclarePairedDelimiterXPP{\probof}[1]{\prob}{(}{)}{}{%		% for conditional probabilities
\renewcommand\given{\nonscript\:\delimsize\vert\nonscript\:\mathopen{}} #1}

\DeclarePairedDelimiterXPP{\oneof}[1]{\one}{\{}{\}}{}{%		% for conditional expectations
\renewcommand\given{\nonscript\,\delimsize\vert\nonscript\,\mathopen{}} #1}

\newcommand{\est}[1]{\hat #1}		% for estimates

\newmacro{\model}{\hat\pay}		% for input signal
\newmacro{\error}{e}		% for error variable
\newmacro{\noise}{z}		% for noise
\newmacro{\bias}{b}		% for bias

\newmacro{\mbound}{M}		% for model bound
\newmacro{\bbound}{B}		% for bias bound
\newmacro{\totbound}{S}		% for bias bound

\newmacro{\snoise}{\psi}		% for scalar noise
\newmacro{\sbias}{\beta}		% for scalar bias

\newmacro{\noisedev}{\sigma}		% for noise stdev
\newmacro{\noisevar}{\noisedev^{2}}		% for noise variance

\newmacro{\unitvar}{E}		% for query direction
\newmacro{\pertvar}{W}		% for perturbation

\newmacro{\radius}{r}		% for Bregman radius

\newmacro{\flowmap}{\Phi}		% for (semi)flows
\newmacro{\graph}{\mathcal{G}}
\newmacro{\vertices}{\mathcal{V}}
\newmacro{\edges}{\mathcal{E}}

\newmacro{\gmat}{g}		% for metric tensor
\newmacro{\gdist}{\dist_{\gmat}}
\newmacro{\ball}{\mathbb{B}}		% for balls
\newmacro{\sphere}{\mathbb{S}}		% for spheres

\newmacro{\const}{c}

\newmacro{\budget}{V}
\newmacro{\batch}{\Delta}
\newmacro{\iBatch}{k}
\newmacro{\nBatches}{m}

\newmacro{\pexp}{\rho}		% for temp exponent
\newmacro{\qexp}{\gamma}		% for batch exponent
\newmacro{\bexp}{\beta}		% for bias exponent
\newmacro{\mexp}{\mu}		% for model exponent
\newmacro{\vexp}{\nu}		% for variation exponent
\newmacro{\wexp}{\mu}		% for variation exponent
\newmacro{\dexp}{\kappa}		% for diameter exponent

\newmacro{\kerfun}{K}		% for kernel function
\newmacro{\width}{\delta}		% for query radius
\newmacro{\mix}{\eps}		% for query radius

\newop{\vbudget}{VB}		% for variation budget
\newcommand{\tvar}[1][\nRuns]{\debug{V_{#1}}}		% for total variation

\begin{document}

%**********************************************************************
%***    FRONT MATTER AND METADATA
%**********************************************************************

%----------------------------------------------------------------------
%%% TITLE & AUTHORS
%----------------------------------------------------------------------
\newcommand{\longtitle}{\uppercase{Online Non-Convex Optimization with Imperfect Feedback}}
\newcommand{\runtitle}{Online Non-Convex Optimization with Imperfect Feedback}		% for run title

\title{\longtitle}		% for title data

%----------------------------------------------------------------------
\author
[A.~Héliou]
{Amélie Héliou$^{\ast}$}
\address{$\ast$\,Criteo AI Lab, France}
\email{a.heliou@criteo.com}

%----------------------------------------------------------------------
\author
[M.~Martin]
{Matthieu Martin$^{\ast}$}
\email{mat.martin@criteo.com}

%-------------------------------------------------------------------
\author
[P.~Mertikopoulos]
{\\Panayotis Mertikopoulos$^{\diamond,\ast}$}
\address{$^{\diamond}$\,%
Univ. Grenoble Alpes, CNRS, Inria, LIG, 38000, Grenoble, France.}
\email{panayotis.mertikopoulos@imag.fr}

%----------------------------------------------------------------------
\author
[T.~Rahier]
{Thibaud Rahier$^{\ast}$}
\email{t.rahier@criteo.com}

%----------------------------------------------------------------------
%%% ACKNOWLEDGMENTS
%----------------------------------------------------------------------
\thanks{P.~Mertikopoulos is grateful for financial support by
the French National Research Agency (ANR) in the framework of
the ``Investissements d'avenir'' program (ANR-15-IDEX-02),
the LabEx PERSYVAL (ANR-11-LABX-0025-01),
and
MIAI@Grenoble Alpes (ANR-19-P3IA-0003).
This research was also supported by the COST Action CA16228 ``European Network for Game Theory'' (GAMENET)\ackperiod
}

%----------------------------------------------------------------------
%%% KEYWORDS
%----------------------------------------------------------------------
\subjclass[2020]{Primary 68Q32; Secondary 90C26, 91A26.}
\keywords{%
Online optimization;
non-convex;
dual averaging;
bandit / imperfect feedback.}

%----------------------------------------------------------------------
%%% ACRONYMS
%----------------------------------------------------------------------
\newacro{LHS}{left-hand side}
\newacro{RHS}{right-hand side}
\newacro{iid}[i.i.d.]{independent and identically distributed}
\newacro{lsc}[l.s.c.]{lower semi-continuous}

\newacro{OGD}{online gradient descent}
\newacro{MD}{mirror descent}
\newacro{OMD}{online mirror descent}
\newacro{MWU}{multiplicative weights update}
\newacro{FTPL}{``follow the perturbed leader''}
\newacro{FTRL}{``follow the regularized leader''}
\newacro{SPSA}{simultaneous perturbation stochastic approximation}
\newacro{RN}{Radon-Nikodym}
\newacro{MCP}{maximum clique problem}
\newacro{DAIM}{dual averaging with inexact models}
\newacro{BDA}{bandit dual averaging}
\newacro{DA}{dual averaging}
\newacro{FTRL}{``follow the regularized leader''}
\newacro{NE}{Nash equilibrium}
\newacroplural{NE}[NE]{Nash equilibria}
\newacro{TA}{test acronym}

%----------------------------------------------------------------------
%%% ABSTRACT
%----------------------------------------------------------------------
\begin{abstract}
%----------------------------------------------------------------------
%%% ABSTRACT
%----------------------------------------------------------------------
% !TEX root = ./Main.tex
%
%
We consider the problem of online learning with non-convex losses.
In terms of feedback, we assume that the learner observes \textendash\ or otherwise constructs \textendash\ an inexact model for the loss function encountered at each stage, and we propose a mixed-strategy learning policy based on dual averaging.
In this general context, we derive a series of tight regret minimization guarantees, both for the learner's static (external) regret, as well as the regret incurred against the best \emph{dynamic} policy in hindsight.
Subsequently, we apply this general template to the case where the learner only has access to the actual loss incurred at each stage of the process.
This is achieved by means of a kernel-based estimator which generates an inexact model for each round's loss function using only the learner's realized losses as input.
\end{abstract}

%**********************************************************************
%***    BODY TEXT
%**********************************************************************
\maketitle
\allowdisplaybreaks		% for breaking long displays
\acresetall

%----------------------------------------------------------------------
%%% INTRODUCTION
%----------------------------------------------------------------------
\section{Introduction}
\label{sec:introduction}
%----------------------------------------------------------------------
%%% INTRODUCTION
%----------------------------------------------------------------------
% !TEX root = ./Main.tex

In this paper, we consider the following online learning framework:
% \vspace{-\smallskipamount}
\begin{enumerate}
%\vspace{-1pt}
\item
At each stage $\run=\running$ of a repeated decision process, the learner selects an action $\choice_{\run}$ from a compact convex subset $\points$ of a Euclidean space $\R^{\vdim}$.
%\vspace{-1pt}
\item
The agent's choice of action triggers a loss $\loss_{\run}(\choice_{\run})$ based on an a priori unknown \emph{loss function} $\loss_{\run}\from\points\to\R$;
subsequently, the process repeats.
\end{enumerate}
If the loss functions $\loss_{\run}$ encountered by the agent are \emph{convex}, the above framework is the standard online convex optimization setting of \citef{Zin03} \textendash\ for a survey, see \citep{SS11,BCB12,Haz12} and references therein.
In this case, simple first-order methods like \ac{OGD} allow the learner to achieve $\bigoh(\nRuns^{1/2})$ regret after $\nRuns$ rounds \cite{Zin03}, a bound which is well-known to be min-max optimal in this setting \citep{ABRT08,SS11}.
At the same time, it is also possible to achieve tight regret minimization guarantees against \emph{dynamic comparators} \textendash\ such as the regret incurred against the best \emph{dynamic policy} in hindsight, \cf \citep{CYLM+12,CBGLS12,BGZ15,HazSes09,JRSS15} and references therein.

On the other hand, when the problem's loss functions are not convex, the situation is considerably more difficult.
When the losses are generated from a stationary stochastic distribution, the problem can be seen as a version of a continuous-armed bandit in the spirit of \citef{Agr95};
in this case, there exist efficient algorithms guaranteeing logarithmic regret by discretizing the problem's search domain and using a UCB-type policy \cite{BMSS11,KSU08,Sli19}.
Otherwise, in an adversarial context, an informed adversary can impose \emph{linear} regret to \emph{any deterministic algorithm} employed by the learner \citep{SS11,HSZ17,SN20};
as a result, UCB-type approaches are no longer suitable.

In view of this impossibility result, two distinct threads of literature have emerged for online non-convex optimization.
One possibility is to examine less demanding measures of regret \textendash\ like the learner's \emph{local regret} \citep{HSZ17} \textendash\ and focus on first-order methods that minimize it efficiently \citep{HSZ17,HalMC20}.
Another possibility is to consider \emph{randomized} algorithms, in which case achieving no regret \emph{is} possible:
\citef{KBTB15} showed that adapting the well-known \emph{Hedge} (or multiplicative\,/\,exponential weights) algorithm to a continuum allows the learner to achieve $\bigoh(\nRuns^{1/2})$ regret, as in the convex case.
This result is echoed in more recent works by \citef{AGH19} and \citef{SN20} who analyzed the \acdef{FTPL} algorithm of \citef{KV05} with exponentially distributed perturbations and an offline optimization oracle (exact or approximate);
again, the regret achieved by \ac{FTPL} in this setting is $\bigoh(\nRuns^{1/2})$, \ie order-equivalent to that of Hedge in a continuum.

%----------------------------------------------------------------------
%%% CONTRIBS
%----------------------------------------------------------------------
\para{Our contributions and related work}

A crucial assumption in the above works on randomized algorithms is that, after selecting an action, the learner receives perfect information on the loss function encountered \textendash\ \ie an \emph{exact model} thereof.
This is an important limitation for the applicability of these methods, which led to the following question by \citef[p.~8]{KBTB15}:
\renewenvironment{quote}{%
\list{}{%
	\leftmargin0.5cm   % this is the adjusting screw
	\rightmargin\leftmargin
}
\item\relax
}
{\endlist}
\begin{quote}
\itshape
\centering
One question is whether one can generalize the Hedge algorithm to a bandit setting, so that sublinear regret can be achieved without the need to explicitly maintain a cover.
\end{quote}

To address this open question, we begin by considering a general framework for randomized action selection with \emph{imperfect feedback} \textendash\ \ie with an inexact model of the loss functions encountered at each stage.
Our contributions in this regard are as follows:
\begin{enumerate}[leftmargin=2em]
\item
We present a flexible algorithmic template for online non-convex learning based on \acl{DA} with imperfect feedback \citep{Nes09}.

\item
We provide tight regret minimization rates \textendash\ both \emph{static} and \emph{dynamic} \textendash\ under a wide range of different assumptions for the loss models available to the optimizer.

\item
We show how this framework can be extended to learning with \emph{bandit feedback}, \ie when the learner only observes their realized loss and must construct a loss model from scratch.
\end{enumerate}

Viewed abstractly, the \ac{DA} algorithm is an ``umbrella'' scheme that contains Hedge as a special case for problems with a simplex-like domain.
%and was first introduced for solving stochastic convex problems by \citef{Nes09};
In the context of online convex optimization, the method is closely related to the well-known \ac{FTRL} algorithm of \citef{SSS07},
the \ac{FTPL} method of \citef{KV05},
``lazy'' \ac{MD} \cite{SS11,BCB12,Bub15},
etc.
%for a survey, see \cite{SSS07,Nes09,SS11,Xia10,BCB12} and references therein.
%The literature surrounding these methods and their variants is too vast to review here;
For an appetizer to the vast literature surrounding these methods, we refer the reader to \cite{SSS07,Nes09,SS11,Xia10,BCB12,BM17,ZMBB+20} and references therein.
%We only mention that the Hedge algorithm can also be seen as an instance of \acl{DA}, a property already exploited by \citef{KBTB15}.

%----------------------------------------------------------------------
%% Related work table begins here

\begin{table}[tbp]
\centering
\small
\renewcommand{\arraystretch}{1.3}
%----------------------------------------------------------------------
%%% PARAMETERS
%----------------------------------------------------------------------
% !TEX root = ../Main.tex

\begin{tabular}{r||lr||lr}
	&\textbf{Convex Losses}
	&
	&\textbf{Non-Convex Losses}
	&
	\\
\textbf{Feedback}
	&Static regret
	&Dynamic regret
	&Static regret
	&Dynamic regret
	\\
\hline
Exact
	&$\bigoh\parens[\big]{\nRuns^{1/2}}$ \; \cite{Zin03}
	&$\bigoh\parens[\big]{\nRuns^{2/3}\tvar^{1/3}}$ \; \cite{BGZ15}
	&$\bigoh\parens[\big]{\nRuns^{1/2}}$ \; \cite{KBTB15,SN20}
	&$\boldsymbol{\bigoh\parens[\big]{\nRuns^{2/3}\tvar^{1/3}}}$
	\\
\hline
Unbiased
	&$\bigoh\parens[\big]{\nRuns^{1/2}}$ \; \cite{Zin03}
	&$\bigoh\parens[\big]{\nRuns^{2/3}\tvar^{1/3}}$ \; \cite{BGZ15}
	&$\boldsymbol{\bigoh\parens[\big]{\nRuns^{1/2}}}$
	&$\boldsymbol{\bigoh\parens[\big]{\nRuns^{2/3}\tvar^{1/3}}}$
	\\
\hline
Bandit
	&$\bigoh\parens[\big]{\nRuns^{1/2}}$ \; \cite{BE16,BLE17}
	&$\bigoh\parens[\big]{\nRuns^{4/5}\tvar^{1/5}}$ \; \cite{BGZ15}
	&$\boldsymbol{\bigoh\parens[\big]{\nRuns^{\frac{\vdim+2}{\vdim+3}}}}$
	&$\boldsymbol{\bigoh\parens[\big]{\nRuns^{\frac{\vdim+3}{\vdim+4}}\tvar^{\frac{1^{\vphantom{1}}}{\vdim+4}}}}$
	\\
\hline
\end{tabular}
\medskip
\caption{%
Overview of related work.
In regards to feedback,
an ``exact'' model means that the learner acquires perfect knowledge of the encountered loss functions;
``unbiased'' refers to an inexact model that is only accurate on average;
finally,
``bandit'' means that the learner records their incurred loss and has no other information.
We only report here the best known bounds in the literature;
all bounds derived in this paper are typeset \textbf{in bold}.
}
\label{tab:related}
\vspace{-2ex}
\end{table}

%% Related work table ends here
%----------------------------------------------------------------------

In the \emph{non-convex} setting, our regret minimization guarantees can be summarized as follows (see also \cref{tab:related} above):
if the learner has access to inexact loss models that are unbiased and finite in mean square, the \ac{DA} algorithm achieves in expectation a \emph{static} regret bound of $\bigoh(\nRuns^{1/2})$.
Moreover, in terms of the learner's \emph{dynamic} regret, the algorithm enjoys a bound of $\bigoh(\nRuns^{2/3}\tvar^{1/3})$ where $\tvar \defeq \insum_{\run=\start}^{\nRuns} \supnorm{\loss_{\run+1} - \loss_{\run}}$ denotes the \emph{variation} of the loss functions encountered over the horizon of play (\cf \cref{sec:general} for the details).
Importantly, both bounds are order-optimal, even in the context of online \emph{convex} optimization, \cf \citep{CBL06,BGZ15,ABRT08}.

With these general guarantees in hand, we tackle the bandit setting using a ``kernel smoothing'' technique in the spirit of \citef{BLE17}.
This leads to a new algorithm, which we call \acdef{BDA}, and which can be seen as a version of the \ac{DA} method with \emph{biased} loss models.
The bias of the loss model can be controlled by tuning the ``radius'' of the smoothing kernel;
however, this comes at the cost of increasing the model's variance \textendash\ an incarnation of the well-known ``bias-variance'' trade-off.
By resolving this trade-off, we are finally able to answer the question of \citef{KBTB15} in the positive:
\ac{BDA} enjoys an $\bigoh(\nRuns^{\frac{\vdim+2}{\vdim+3}})$ static regret bound and an $\bigoh(\nRuns^{\frac{\vdim+3}{\vdim+4}}\tvar^{1/(\vdim+4)})$ \emph{dynamic} regret bound, without requiring an explicit discretization of the problem's search space.

This should be contrasted with the case of online \emph{convex} learning, where it is possible to achieve $\bigoh(\nRuns^{3/4})$ regret through the use of \ac{SPSA} techniques \citep{FKM05}, or even $\bigoh(\nRuns^{1/2})$ by means of kernel-based methods \citep{BE16,BLE17}.
This represents a drastic drop from $\bigoh(\nRuns^{1/2})$, but this cannot be avoided:
the worst-case bound for stochastic non-convex optimization is $\Omega(\nRuns^{(\vdim+1)/(\vdim+2)})$ \citep{Kle04,KSU08}, so our static regret bound is nearly optimal in this regard (\ie up to $\bigoh(\nRuns^{-1/(\vdim+2)(\vdim+3)})$, a term which is insignificant for horizons $\nRuns\leq10^{12}$).
%We are not aware of a method achieving this bound without an explicit discretization of the state space \cite{Kle04}.
Correspondingly, in the case of dynamic regret minimization, the best known upper bound is $\bigoh(\nRuns^{4/5}\tvar^{1/5})$ for online \emph{convex} problems \citep{BGZ15,DMSV18}.
We are likewise not aware of any comparable dynamic regret bounds for online \emph{non-convex} problems;
to the best our knowledge, our paper is the first to derive dynamic regret guarantees for online non-convex learning with bandit feedback.

We should stress here that, as is often the case for methods based on lifting, much of the computational cost is hidden in the sampling step.
This is also the case for the proposed \ac{DA} method which, like \cite{KBTB15}, implicitly assumes access to a sampling oracle.
Estimating (and minimizing) the per-iteration cost of sampling is an important research direction, but one that lies beyond the scope of the current paper, so we do not address it here.

%----------------------------------------------------------------------
%%% SETUP
%----------------------------------------------------------------------
\section{Setup and preliminaries}
\label{sec:setup}
%----------------------------------------------------------------------
%%% SETUP
%----------------------------------------------------------------------
% !TEX root = ./Main.tex

%----------------------------------------------------------------------
%%% MODEL
%----------------------------------------------------------------------
\subsection{The model}

Throughout the sequel, our only blanket assumption will be as follows:
\begin{assumption}
\label{asm:loss}
The stream of loss functions encountered is \emph{uniformly bounded Lipschitz},
\ie there exist constants $\vbound, \lips > 0$ such that:
\begin{enumerate}
\setlength{\itemsep}{0pt}
\item
%There exists $\vbound \geq 0$ such that
$\abs{\loss_{\run}(\point)} \leq \vbound$ for all $\point\in\points$; more succinctly, $\supnorm{\loss_{\run}} \leq \vbound$.
\item
%There exists $\lips \geq 0$ such that
$\abs{\loss_{\run}(\pointalt) - \loss_{\run}(\point)} \leq \lips \norm{\pointalt - \point}$
for all $\point,\pointalt\in\points$.
\end{enumerate}
\end{assumption}

Other than this meager regularity requirement, we make no structural assumptions for $\loss_{\run}$ (such as convexity, unimodality, or otherwise).
In this light, the framework under consideration is akin to the online non-convex setting of \citef{KBTB15}, \citef{HSZ17}, and \citef{SN20}.
The main difference with the setting of \citef{KBTB15} is that the problem's domain $\points$ is assumed convex;
this is done for convenience only, to avoid technical subtleties involving ``uniform fatness'' conditions and the like.
%and significantly more general than the standard online \emph{convex} framework of \citef{Zin03}.
%the framework under consideration is considerably more general than the standard online convex optimization setting of \citef{Zin03}.
%instead, modulo some light technical details, our framework is otherwise similar to the online non-convex setup of \citef{KBTB15}, \citef{HSZ17}, and \citef{SN20}.

In terms of playing the game, we will assume that the learner can employ \emph{mixed strategies} to randomize their choice of action at each stage;
however, because this mixing occurs over a \emph{continuous} domain, defining this randomization requires some care.
To that end, let $\measures \equiv \measures(\points)$ denote the space of all finite signed Radon measures on $\points$.
Then, a \emph{mixed strategy} is defined as an element $\strat$ of the set of Radon probability measures $\strats \equiv \strats(\points) \subseteq \measures(\points)$ on $\points$,
%This space becomes a Banach space when endowed with the total variation norm $\tvnorm{\meas} = \meas^{+}(\points) + \meas^{-}(\points)$, where $\meas^{+}$ (resp. $\meas^{-}$) denotes the positive (resp. negative) part of $\meas$ coming from the Hahn-Banach decomposition of signed measures on $\points$ \citep{Fol99}.
%In this topology, the set of Radon probability measures on $\points$ forms a closed subset $\strats \equiv \strats(\points)$ of $\measures$;
%clearly $\tvnorm{\strat} = \int_{\points} \dd\strat(\point) = 1$ for all $\strat \in \strats$.
%%$\strat(A) \geq 0$ for every Borel $A\subseteq\points$.
%A \emph{mixed strategy} is then defined to be an element $\strat$ of $\strats$,
and the player's expected loss under $\strat$ when facing a bounded loss function $\loss\in\linf$ will be denoted as
%\footnote{In the above, $\linf$ denotes the space of all essentially bounded Lebesgue-measurable functions on $\points$.
%We also note here that $\measures$ is a Banach space under the total variation norm $\tvnorm{\meas} = \meas^{+}(\points) + \meas^{-}(\points)$, where $\meas^{\pm}$ denotes the $\{\}$-ve part of $\meas$ coming from the Hahn-Banach decomposition of signed measures on $\points$ \citep{Fol99}.
%In this topology, $\strats$ is a closed subset of $\measures$ and $\tvnorm{\strat} = \int_{\points} \dd\strat(\point) = 1$ for all $\strat \in \strats$.}
\begin{equation}
\txs
\braket{\loss}{\strat}
	\defeq \ex_{\strat}[\loss] = \int_{\points} \loss(\point) \dd\strat(\point).
\end{equation}

\begin{remark}
We should note here that $\strats$ contains a vast array of strategies, including atomic and singular distributions that do not admit a density.
For this reason, we will write $\strats_{\abscont}$ for the set of strategies that are \emph{absolutely continuous} relative to the Lebesgue measure $\leb$ on $\points$, and $\strats_{\sing}$ for the set of singular strategies (which are not);
by Lebesgue's decomposition theorem \citep{Fol99}, we have $\strats = \strats_{\abscont} \cup \strats_{\sing}$.
By construction, $\strats_{\sing}$ contains the player's \emph{pure strategies}, \ie Dirac point masses $\dirac_{\point}$ that select $\point\in\points$ with probability $1$;
however, it also contains pathological strategies that admit \emph{neither} a density, \emph{nor} a point mass function \textendash\ such as the Cantor distribution \citep{Fol99}.
By contrast, the \ac{RN} derivative $\pdist \defeq d\strat/d\leb$ of $\strat$ exists for all $\strat\in\strats_{\abscont}$, so we will sometimes refer to elements of $\strats_{\abscont}$ as ``\acl{RN} strategies'';
in particular, if $\strat\in\strats_{\abscont}$, we will not distinguish between $\strat$ and $\pdist$ unless absolutely necessary to avoid confusion.
\end{remark}

Much of our analysis will focus on strategies $\simple$ with a piecewise constant density on $\points$, \ie $\simple = \sum_{i=1}^{k} \coef_{i} \one_{\cvx_{i}}$ for a collection of
weights $\coef_{i}\geq0$
and
measurable subsets $\cvx_{i}\subseteq\points$,
$i=1,\dotsc,k$, such that $\int_{\points} \simple = \sum_{i} \coef_{i} \leb(\cvx_{i}) = 1$.
These strategies will be called \emph{simple} and the space of simple strategies on $\points$ will be denoted by $\simples \equiv \simples(\points)$.
A key fact regarding simple strategies is that
%First, since they all admit a finite sum representation, sampling from simple strategies can be done efficiently (especially if the geometry of $\points$ is relatively simple).
%Second, by standard arguments in functional analysis,
$\simples$ is dense in $\strats$ in the weak topology of $\measures$ \cite[Chap.~3]{Fol99};
%\TR{\tiny Enlightening way of understanding why our step-kernel approach is cool}
%\PM{\smiley}
as a result, the learner's expected loss under \emph{any} mixed strategy $\strat\in\strats$ can be approximated within arbitrary accuracy $\eps>0$ by a simple strategy $\simple\in\simples$.
In addition, when $k$ (or $\vdim$) is not too large, sampling from simple strategies can be done efficiently;
%by means of a spline, Fourier or wavelet basis;
for all these reasons, simple strategies will play a key role in the sequel.

%----------------------------------------------------------------------
%%% REGRET
%----------------------------------------------------------------------
\subsection{Measures of regret}

With all this in hand,
%a \emph{learning policy} is a sequence of mixed strategies $\policy_{\run}\in\strats$.
the \emph{regret} of a learning policy $\policy_{\run}\in\strats$, $\run=\running$, against a benchmark strategy $\bench\in\strats$ is defined as
%\TR{In terms of vocabulary, am I right that `policy' designates a \textbf{sequence} of `strategies' for $\run=\running$? By `benchmark strategy' $\strat$, we mean a policy that corresponding to a fixed strategy through time? Pbly splitting hairs}
%\PM{Correct \textendash\ changed notation to make things more clear.}
\begin{equation}
\label{eq:reg-test}
\txs
\reg_{\bench}(\nRuns)
	= \insum_{\run=\start}^{\nRuns} \bracks[\big]{ \ex_{\policy_{\run}}[\loss_{\run}] - \ex_{\bench}[\loss_{\run}]}
	= \insum_{\run=\start}^{\nRuns} \braket{\loss_{\run}}{\policy_{\run} - \bench},
\end{equation}
%\TR{Just to be sure: $\strat$ in subscript in $\reg_{\strat}(\nRuns)$ refers to the benchmark constant policy $\strat$ and not to the policy we want to evaluate $\policy_\run$? This would mean that $\policy_\run$ does not `appear' in the regret notation but maybe that is not a pb at all}
%\PM{Correct, $\policy_{\run}$ doesn't appear in the definition of the regret.
%Will change the notation for the benchmark to make sure there's no confusion \smiley}
\ie as the difference between the player's mean cumulative loss under $\policy_{\run}$ and $\bench$ over $\nRuns$ rounds.
In a slight abuse of notation, we write
$\reg_{\benchdist}(\nRuns)$ if $\bench$ admits a density $\benchdist$,
and
$\reg_{\point}(\nRuns)$ for the regret incurred against the pure strategy $\dirac_{\point}$, $\point\in\points$.
Then, the player's \emph{\textpar{static} regret} under $\policy_{\run}$ is given by
\begin{equation}
\label{eq:reg}
\txs
\reg(\nRuns)
	= \max_{\point\in\points} \reg_{\point}(\nRuns)
	= \sup_{\simple\in\simples} \reg_{\simple}(\nRuns)
\end{equation}
where the maximum is justified by the compactness of $\points$ and the continuity of each $\loss_{\run}$.
%\footnote{The second equality is justified by the fact that $\simples$ is dense in $\strats$ in the weak topology on $\measures$.}
The lemma below provides a link between pure comparators and their approximants in the spirit of \citef{KBTB15};
to streamline our discussion, we defer the proof to the supplement:

\begin{restatable}{lemma}{pointsimple}
\label{lem:point2simple}
Let $\nhd$ be a convex neighborhood of $\point$ in $\points$ and let $\simple\in\simples$ be a simple strategy supported on $\nhd$.
Then, $\reg_{\point}(\nRuns) \leq \reg_{\simple}(\nRuns) + \lips\diam(\nhd) \nRuns$.
\end{restatable}

%\begin{proof}
%By \cref{asm:loss}, we have $\loss_{\run}(\point) \leq \loss_{\run}(\pointalt) + \lips \norm{\point - \pointalt} \leq \loss_{\run}(\pointalt) + \lips \diam(\nhd)$ for all $\pointalt\in\nhd$.
%Hence, taking expectations on both sides relative to $\simple$, we get $\loss_{\run}(\point) \leq \braket{\loss_{\run}}{\simple} + \lips\diam(\nhd)$.
%Our claim then follows by summing over $\run=\running,\nRuns$ and invoking the definition of the regret.
%\end{proof}

This lemma will be used to bound the agent's static regret using bounds obtained for simple strategies $\simple\in\simples$.
Going beyond static comparisons of this sort, the learner's \emph{dynamic regret} is defined as
\begin{equation}
\label{eq:dynreg}
\txs
\dynreg(\nRuns)
	= \insum_{\run=\start}^{\nRuns} \bracks{ \braket{\loss_{\run}}{\strat_{\run}} - \min_{\strat\in\strats} \braket{\loss_{\run}}{\strat}}
	= \insum_{\run=\start}^{\nRuns} \braket{\loss_{\run}}{\strat_{\run} - \sol[\strat]_{\run}}
\end{equation}
where $\sol[\strat]_{\run} \in \argmin_{\strat\in\strats} \braket{\loss_{\run}}{\strat}$ is a ``best-response'' to $\loss_{\run}$ (that such a strategy exists is a consequence of the compactness of $\points$ and the continuity of each $\loss_{\run}$).
In regard to its static counterpart, the agent's dynamic regret is considerably more ambitious, and achieving sublinear dynamic regret is not always possible;
we examine this issue in detail in \cref{sec:general}.

%----------------------------------------------------------------------
%%% FEEDBACK
%----------------------------------------------------------------------
\subsection{Feedback models}

After choosing an action, the agent is only assumed to observe an \emph{inexact model} $\model_{\run} \in \linf$ of the $\run$-th stage loss function $\loss_{\run}$;
for concreteness, we will write
\begin{equation}
\label{eq:model}
\model_{\run}
	= \loss_{\run}
	+ \error_{\run}
\end{equation}
where
the ``observation error'' $\error_{\run}$ captures all sources of uncertainty in the player's model.
This uncertainty could be both ``random'' (zero-mean) or ``systematic'' (non-zero-mean), so it will be convenient to decompose $\error_{\run}$ as
\begin{equation}
\label{eq:error}
\error_{\run}
	= \noise_{\run} + \bias_{\run}
\end{equation}
where $\noise_{\run}$ is zero-mean and $\bias_{\run}$ denotes the mean of $\error_{\run}$.

To define all this formally, we will write $\filter_{\run} = \history(\policy_{1},\dotsc,\policy_{\run})$ for the history of the player's mixed strategy up to stage $\run$ (inclusive).
The chosen action $\choice_{\run}$ and the observed model $\model_{\run}$ are both generated \emph{after} the player chooses $\policy_{\run}$ so, by default, they are not $\filter_{\run}$-measurable.
Accordingly, we will collect all randomness affecting $\model_{\run}$ in an abstract probability law $\prob$, and we will write
$\bias_{\run} = \exof{\error_{\run} \given \filter_{\run}}$
and
$\noise_{\run} = \error_{\run} - \bias_{\run}$;
%\begin{equation}
%\bias_{\run}
%	= \exof{\error_{\run} \given \filter_{\run}}
%	\qquad
%	\text{and}
%	\qquad
%\noise_{\run}
%	= \error_{\run} - \bias_{\run}
%\end{equation}
in this way, $\exof{\noise_{\run} \given \filter_{\run}} = 0$ by definition.

In view of all this, we will focus on the following descriptors for $\model_{\run}$:
\begin{subequations}
\label{eq:model-stats}
\begin{alignat}{3}
\label{eq:bias}
a)\quad
	&\textit{Bias:}
	&\supnorm{\bias_{\run}}
		&\leq \bbound_{\run}
	\\
b)\quad
\label{eq:variance}
	&\textit{Variance:}
	&\exof{\supnorm{\noise_{\run}}^{2} \given \filter_{\run}}
		&\leq \sdev_{\run}^{2}
	\\
c)\quad
\label{eq:moment}
	&\textit{Mean square:}
	\hspace{2em}
	&\exof{\supnorm{\model_{\run}}^{2} \given \filter_{\run}}
		&\leq \mbound_{\run}^{2}
	\hspace{16em}
\end{alignat}
%Finally, to simplify notation later on, we will also consider the ``signal plus noise'' bound
%\begin{equation}
%\label{eq:totbound}
%\totbound_{\run}^{2}
%	= \mbound_{\run}^{2} + \sdev_{\run}^{2}.
%\end{equation}
\end{subequations}
%By assumption, $\exof{\noise_{\run} \given \filter_{\run}} = 0$ for all $\run$.
In the above,
%$\dnorm{\cdot}$ is a reference norm on the space $\linf$ of continuous functions on $\points$,
$\bbound_{\run}$, $\sdev_{\run}$ and $\mbound_{\run}$ are deterministic constants that are to be construed as bounds on the bias, (conditional) variance, and magnitude of the model $\model_{\run}$ at time $\run$.
In obvious terminology, a model with $\bbound_{\run}=0$ will be called \emph{unbiased},
%and a model with $\lim_{\run\to0}\bbound_{\run} = 0$ will be called \emph{asymptotically unbiased};
and an unbiased model with $\sdev_{\run} = 0$  will be called \emph{exact}.

\smallskip
\begin{example}[Parametric models]
An important application of online optimization is the case where the encountered loss functions are of the form $\loss_{\run}(\point) = \loss(\point;\param_{\run})$ for some sequence of parameter vectors $\param_{\run} \in \R^{m}$.
In this case, the learner typically observes an estimate $\est\param_{\run}$ of $\param_{\run}$, leading to the inexact model $\model_{\run} = \loss(\cdot;\est\param_{\run})$.
Importantly, this means that $\model_{\run}$ \emph{does not require infinite-dimensional feedback} to be constructed.
Moreover, the dependence of $\loss$ on $\param$ is often linear, so if $\est\param_{\run}$ is an unbiased estimate of $\param_{\run}$, then so is $\model_{\run}$.
\endenv
\end{example}

\smallskip
\begin{example}[Online clique prediction]
As a specific incarnation of a parametric model, consider the problem of finding the largest complete subgraph \textendash\ a \emph{maximum clique} \textendash\ of an undirected graph $\graph=(\vertices,\edges)$.
This is a key problem in machine learning with applications to social networks \citep{For10}, data mining \citep{BBH11}, gene clustering \citep{SM03}, feature embedding \citep{ZSYZ+18}, and many other fields.
In the online version of the problem, the learner is asked to predict such a clique in a graph $\graph_{\run}$ that evolves over time (\eg a social network), based on partial historical observations of the graph.
Then, by the Motzkin-Straus theorem \citep{MS65,BBPP99}, this boils down to an online quadratic program of the form:
\begin{equation}
\label{eq:MCP}
\tag{MCP}
\txs
\textrm{maximize}
	\;\;
	\pay_{\run}(\point)
		= \sum_{i,j=1}^{\vdim} \point_{i} A_{ij,\run} \point_{j}
	\qquad
\textrm{subject to}
	\;\;
	\point_{i}\geq0,\;
	\sum_{i=1}^{\vdim} \point_{i} = 1,
\end{equation}
where $A_{\run} = (A_{ij,\run})_{i,j=1}^{\vdim}$ denotes the adjacency matrix of $\graph_{\run}$.
Typically, $\est A_{\run}$ is constructed by picking a node $i$ uniformly at random, charting out its neighbors, and letting $\est A_{ij,\run} = \abs{\vertices}/2$ whenever $j$ is connected to $i$.
It is easy to check that $\est A_{\run}$ is an unbiased estimator of $A_{\run}$;
as a result, the function $\est\pay_{\run}(\point) = \point^{\top} \est A_{\run} \point$ is an unbiased model of $\pay_{\run}$.
%As a result, if the learner has a partial observation $\est A_{\run}$ of the true graph $A_{\run}$ after making a prediction $\choice_{\run}$ at time $\run$, an inexact model for $\pay_{\run}$ is obtained by taking $\est\pay_{\run}(\point) = \point^{\top} \est A_{\run} \point$.
%Obviously,
%$\est\pay_{\run} - \pay_{\run} = \bigoh(\est A_{\run} - A_{\run})$, so the statistical descriptors \labelcref{eq:bias,eq:variance,eq:moment} of $\est\pay_{\run}$ can be expressed in terms of the bias and variance of the estimate of $A_{\run}$.
\endenv
\end{example}

\smallskip
\begin{example}[Online-to-batch]
Consider an empirical risk minimization model of the form
\begin{equation}
\label{eq:risk}
\txs
\obj(\point)
	= \frac{1}{m} \sum_{i=1}^{m} \obj_{i}(\point)
\end{equation}
where each $\obj_{i}\from\R^{\vdim}\to\R$ corresponds to a data point (or ``sample'').
In the ``online-to-batch'' formulation of the problem \cite{SS11}, the optimizer draws uniformly at random a sample $i_{t} \in \{1,\dotsc,m\}$ at each stage $\run=\running$, and observes $\est\loss_{\run} = \obj_{i_{\run}}$. Typically, each $\obj_{i}$ is relatively easy to store in closed form, so $\est\loss_{\run}$ is an easily available unbiased model of the empirical risk function $\obj$.
\endenv
\end{example}

%To streamline our presentation, we first present our results in an abstract setting, \ie without specifying the origins of the inexact model \eqref{eq:model};
%subsequently, in \cref{sec:bandit}, we detail the construction of such a model from bandit, payoff-based observations.

%----------------------------------------------------------------------
%%% DUAL AVERAGING
%----------------------------------------------------------------------
\section{Prox-strategies and dual averaging}
\label{sec:DA}
%----------------------------------------------------------------------
%%% DA
%----------------------------------------------------------------------
% !TEX root = ./Main.tex

The class of non-convex online learning policies that we will consider is based on the general template of \acdef{DA} / \acdef{FTRL} methods.
%This algorithmic framework has been studied extensively in convex optimization, both offline and online;
Informally, this scheme can be described as follows:
at each stage $\run=\running$, the learner plays a mixed strategy that minimizes their cumulative loss up to round $\run-1$ (inclusive) plus a ``regularization'' penalty term (hence the ``regularized leader'' terminology).
In the rest of this section, we provide a detailed construction and description of the method.

%The basic ingredient of this class of methods is the concept of a ``regularizer function'', which allows the learner to generate strategy recommendations from past loss information (\eg an ensemble of inexact models in our case).

%----------------------------------------------------------------------
%%% INFINITE-DIM
%----------------------------------------------------------------------
\subsection{Randomizing over discrete \vs continuous sets}

We begin by describing the \acl{DA} method when the underlying action set is \emph{finite}, \ie of the form $\pures = \{1,\dotsc,\nPures\}$.
%a fundamental technical issue that arises when randomizing over a continuum of choices.
%When randomizing over a finite set $\pures = \{1,\dotsc,\nPures\}$,
In this case, the space of mixed strategies is the $\vdim$-dimensional simplex $\simplex_{\nPures} \defeq \simplex(\pures) = \setdef{\pdist\in\R_{+}^{\nPures}}{\sum_{i=1}^{\nPures} \pdist_{i} = 1}$, and, at each $\run=\running$, the \acl{DA} algorithm
%is straightforward to define.
%Modulo some technical details, at each stage $\run=\running$, \acl{DA}
prescribes the mixed strategy
\begin{equation}
\label{eq:choice-disc}
\txs
\state_{\run}
	\gets \argmin_{\pdist\in\simplex_{\nPures}}
		\{\temp \sum_{\runalt=\start}^{\run-1} \braket{\model_{\run}}{\pdist} + \hreg(\pdist)\}.
\end{equation}
In the above, $\temp>0$ is a ``learning rate'' parameter and $\hreg\from\simplex_{\nPures} \to \R$ is the method's ``regularizer'', assumed to be continuous and strongly convex over $\simplex_{\nPures}$.
In this way, the algorithm can be seen as tracking the ``best'' choice up to the present, modulo a ``day 0'' regularization component \textendash\ the \acl{FTRL} interpretation.

In our case however, the method is to be applied to the \emph{infinite-dimensional} set $\strats \equiv \strats(\points)$ of the learner's mixed strategies, so the issue becomes considerably more involved.
%since $\strats$ is infinite-dimensional, the construction of regularizers is a much more subtle affair.
To illustrate the problem, consider one of the prototypical regularizer functions, the negentropy $\hreg(\primal) = \sum_{i=1}^{\vdim} \primal_{i} \log\primal_{i}$ on $\simplex_{\vdim}$.
%This regularizer is defined on the entire simplex and is $1$-strongly convex relative to the $\ell_{1}$-norm on $\simplex_{\vdim}$.
If we naïvely try to extend this definition to the infinite-dimensional space $\strats(\points)$, we immediately run into problems:
First, for pure strategies, any expression of the form $\sum_{\pointalt\in\points} \delta_{\point}(\pointalt) \log\delta_{\point}(\pointalt)$ would be meaningless.
Second, even if we focus on \acl{RN} strategies $\primal\in\strats_{\abscont}$ and use the integral definition $\hreg(\primal) = \int_{\points} \primal \log\primal$, a density like $\primal(\point) \propto 1/(\point(\log\point)^{2})$ on $\points = [0,1/2]$ has \emph{infinite} negentropy, implying that even $\strats_{\abscont}$ is too large to serve as a domain.

%----------------------------------------------------------------------
%%% CONSTRUCTION
%----------------------------------------------------------------------
\subsection{Formal construction of the algorithm}

%Our starting point below will be that any infinite-dimensional incarnation of the \acl{DA} algorithm must contain \emph{at least} the space $\simples \equiv \simples(\points)$ of the player's simple strategies.
%As we explained in \cref{sec:setup}, these strategies are among the easiest to sample from (since their density is piecewise constant), so they will form the core of our considerations.

To overcome the issues identified above,
our starting point will be that any mixed-strategy incarnation of the \acl{DA} algorithm must contain \emph{at least} the space $\simples \equiv \simples(\points)$ of the player's simple strategies.
To that end, let $\vecspace$ be an ambient Banach space which contains the set of simple strategies $\simples$ as an embedded subset.
For technical reasons, we will also assume that the topology induced on $\simples$ by the reference norm $\norm{\cdot}$ of $\vecspace$ is not weaker than the natural topology on $\simples$ induced by the total variation norm;
formally, $\tvnorm{\cdot} \leq \coef \norm{\cdot}$ for some $\coef>0$.%
\footnote{Since the dual space of $\measures(\points)$ contains $\linf$, we will also view $\linf$ as an embedded subset of $\dspace$.}
%(\ie $\norm{\cdot}_{1} \leq \coef \norm{\cdot}$ on $\simples$ for some $\coef>0$).
For example, $\vecspace$ could be the (Banach) space $\measures(\points)$ of finite signed measures on $\points$,
%equipped with the total variation norm $\tvnorm{\cdot}$,
the (Hilbert) space $\squareints$ of square integrable functions on $\points$ endowed with the $L^{2}$ norm,%
\footnote{In this case, $\coef = \sqrt{\leb(\points)}$:
this is because $\tvnorm{\pdist}^{2} = \bracks*{ \int_{\points}\pdist }^{2}\leq \int_{\points} \pdist^{2} \cdot \int_{\points} 1 = \leb(\points) \twonorm{\pdist}^{2}$ if $\pdist\in\strats_{\abscont}$.}
or an altogether different model for $\simples$.
We then have:

\begin{definition}
\label{def:breg}
A \emph{regularizer} on $\simples$ is a \ac{lsc} convex function $\hreg\from\vecspace\to\R\cup\{\infty\}$ such that:
\vspace{-\smallskipamount}
\begin{enumerate}
%[(\itshape a\upshape)]
\setlength{\itemsep}{0pt}
\item
$\simples$ is a weakly dense subset of the \emph{effective domain} $\dom\hreg \defeq \setdef{\primal}{\hreg(\primal) < \infty}$ of $\hreg$.
%it is dense in the weak topology of $\vecspace$.
%\item
%$\hreg(\primal)\to\infty$ as $\norm{\primal}\to\infty$ (\ie $\hreg$ is \emph{coercive}).
\item
The subdifferential $\subd\hreg$ of $\hreg$ admits a \emph{continuous selection}, \ie there exists a continuous mapping $\nabla\hreg$ on $\dom\subd\hreg \defeq \{\subd\hreg \neq \varnothing\}$ such that $\nabla\hreg(\proxal) \in \subd\hreg(\proxal)$ for all $\proxal\in\dom\subd\hreg$.
\item
$\hreg$ is strongly convex, \ie there exists some $\hstr>0$ such that $\hreg(\primal) \geq \hreg(\proxal) + \braket{\nabla\hreg(\proxal)}{\primal - \proxal} + (\hstr/2) \norm{\primal - \proxal}^{2}$ for all $\primal\in\dom\hreg$, $\proxal\in\dom\subd\hreg$.
\end{enumerate}
The set $\proxdom \defeq \dom\subd\hreg$ will be called the \emph{prox-domain} of $\hreg$;
its elements will be called \emph{prox-strategies}.
\end{definition}
%\TR{wow}

\begin{remark*}
%The set $\proxdom \defeq \dom\subd\hreg$ will be called the \emph{prox-domain} of $\hreg$ and elements of $\proxdom$ will be called \emph{prox-strategies}.
For completeness, recall that the subdifferential of $\hreg$ at $\proxal$ is the set $\subd\hreg(\proxal) = \setdef{\dual\in\dspace}{\hreg(\primal) \geq \hreg(\proxal) + \braket{\dual}{\primal - \proxal} \text{ for all } \proxal\in\vecspace}$;
also
lower semicontinuity means that the sublevel sets $\{\hreg \leq c\}$ of $\hreg$ are closed for all $c\in\R$.
For more details, we refer the reader to \citef{Phe93}.
\end{remark*}

%However, since we are working with randomized strategies, this regularization must be applied directly to the space $\simples$ of the learner's smooth strategies.
%Given that this space is infinite-dimensional, we will focus for concreteness on regularizers of the general form $\hreg(\simple) = \int_{\points} \hdec(\simple(\point)) \dd\point$ where the so-called \emph{Legendre kernel} $\hdec\from[0,\infty)\to\R\cup\{\infty\}$ of $\hreg$ is assumed to have
%\begin{enumerate*}
%[(\itshape a\upshape)]
%\item
%$\hdec(s) < \infty$ for all $s>0$;
%\item
%$\lim_{s\to0^{+}} \hdec'(s) = -\infty$;
%and
%\item
%$\hdec''(s) > 0$ for all $s>0$.
%\end{enumerate*}

Some prototypical examples of this general framework are as follows (with more in the supplement):

\smallskip
\begin{example}[$L^{2}$ regularization]
\label{ex:L2}
Let $\vecspace = \squareints$ and consider the quadratic regularizer $\hreg(\primal) = (1/2) \norm{\primal}_{2}^{2} = (1/2) \int_{\points} \primal^{2}$
if $\primal \in \strats_{\abscont} \cap \squareints$,
%for $\primal \in \setdef{\proxal \in \squareints}{\proxal\geq0,\int_{\points}\proxal = 1}$;
and $\hreg(\primal) = \infty$ otherwise.
In this case, $\proxdom = \dom\hreg = \strats_{\abscont} \cap \squareints$ and $\nabla\hreg(\proxal) = \proxal$ is a continuous selection of $\subd\hreg$ on $\proxdom$.
\endenv
\end{example}

\smallskip
\begin{example}[Entropic regularization]
\label{ex:entropy}
Let $\vecspace = \measures(\points)$ and consider the entropic regularizer $\hreg(\primal) = \int_{\points} \primal \log\primal$ whenever $\primal$ is a density with finite entropy, $\hreg(\primal)=\infty$ otherwise.
By Pinsker's inequality, $\hreg$ is $1$-strongly convex relative to the total variation norm $\tvnorm{\cdot}$ on $\vecspace$;
moreover, we have $\proxdom = \setdef{\proxal\in\strats_{\abscont}}{\supp(\proxal) = \points} \subsetneq \dom\hreg$ and $\nabla\hreg(\proxal) = 1 + \proxal \log\proxal$ on $\proxdom$.
In the finite-dimensional case, this regularizer forms the basis of the well-known Hedge (or multiplicative/exponential weights) algorithm \citep{Vov90,LW94,ACBFS95,AHK12};
for the infinite-dimensional case, see \cite{KBTB15,PML17} (and below).
\endenv
\end{example}

%\begin{example}[Log-barrier regularization]
%\label{ex:barrier}
%Let $\vecspace = \measures(\points)$ as above and consider the so-called \emph{Burg entropy} $\hreg(\primal) = -\int_{\points} \log\primal$ \citep{ABB04}.
%This regularizer plays a fundamental role in the affine scaling method of \citef{Kar90}, see \eg \citef{Tse04,VMF86} and references therein.
%\endenv
%\end{example}

With all this in hand, the \acl{DA} algorithm can be described by means of the abstract recursion
\begin{equation}
\label{eq:DA}
\tag{DA}
%\begin{aligned}
\score_{\run+1}
	= \score_{\run} - \model_{\run},
%	\\
\quad
\state_{\run+1}
	= \mirror(\temp_{\run+1} \score_{\run+1}),
%\end{aligned}
\end{equation}
where
%\vspace{-\smallskipamount}
\begin{enumerate*}
[(\itshape i\hspace*{.5pt}\upshape)]
\item
$\run=\running$ denotes the stage of the process (with the convention $\score_{0} = \model_{0} = 0$);
\item
$\state_{\run} \in \proxdom$ is the learner's strategy at stage $\run$;
\item
$\model_{\run} \in \linf$ is the inexact model revealed at stage $\run$;
\item
$\score_{\run} \in \linf$ is a ``score'' variable that aggregates loss models up to stage $\run$;
\item
$\temp_{\run} > 0$ is a ``learning rate'' sequence;
%\footnote{Note here that $\temp_{\run}$ should not be seen as a step-size sequence because it does not intervene at each model $\model_{\run}$.}
and
\item
$\mirror \from \linf \to \proxdom$ is the method's \emph{mirror map}, viz.
\end{enumerate*}
\begin{equation}
\label{eq:mirror}
\txs
\mirror(\dual)
	= \argmax_{\primal\in\vecspace} \{ \braket{\dual}{\primal} - \hreg(\primal) \}.
\end{equation}

%----------------------------------------------------------------------
%% DA pseudocode begins here

\begin{algorithm}[tbp]
\ttfamily
\small
\caption{Dual averaging with imperfect feedback \hfill [Hedge variant: $\mirror \subs \logit$]}
%----------------------------------------------------------------------
%%% DA ALGORITHM
%----------------------------------------------------------------------
% !TEX root = ../Main.tex

\begin{algorithmic}[1]
\Require
	mirror map $\mirror\from\linf\to\proxdom$;
	learning rate $\temp_{\run}>0$;
	\textbf{initialize:} $\score_{\start} \subs 0$
%\State
%	choose $\score_{\start} = 0$
%%	play $\State \subs \mirror(\score) = \argmin\hreg$
%	\Comment{initialization}%
\For{$\run=\running$}
	\State
		set $\state_{\run} \subs \mirror(\temp_{\run}\score_{\run})$
		\qquad
		[$\state_{\run} \subs \logit(\temp_{\run}\score_{\run})$ for Hedge]
		\Comment{update mixed strategy}%
	\State
		play $\choice_{\run} \sim \state_{\run}$
		\Comment{choose action}%
	\State
		observe $\model_{\run}$
%		receive $\pay_{\run}(\choice_{\run})$
%		receive $\pay_{\run}(\choice_{\run})$ and observe $\model_{\run}$
		\Comment{model revealed}%
	\State
		set $\score_{\run+1} \subs \score_{\run} - \model_{\run}$
		\Comment{update scores}%
\EndFor
%\State
%	\Return $\State$
\end{algorithmic}
\label{alg:DA}
\end{algorithm}

%% DA pseudocode ends here
%----------------------------------------------------------------------

%Implicit in the above is the stipulation that, at each $\run=\running$, the learner chooses an action $\choice_{\run} \in \points$ according to $\state_{\run} \in \proxdom$;
For a pseudocode implementation, see \cref{alg:DA} above.
In the paper's supplement we also show that the method is \emph{well-posed}, \ie the $\argmax$ in \eqref{eq:mirror} is attained at a valid prox-strategy $\state_{\run} \in \proxdom$.
We illustrate this with an example:

\begin{example}[Logit choice]
\label{ex:logit}
Suppose that $\hreg(\primal) = \int_{\points} \primal\log\primal$ is the entropic regularizer of \cref{ex:entropy}.
Then,
%a calculation based on standard arguments shows that
the corresponding mirror map is given in closed form by the logit choice model:
\begin{equation}
\label{eq:logit}
\logit(\fun)
	= \frac{\exp(\fun)}{\int_{\points} \exp(\fun)}
	\quad
	\text{for all $\fun\in\linf$}.
\end{equation}
This derivation builds on a series of well-established arguments that we defer to the supplement.
Clearly, $\int_{\points} \logit(\fun) = 1$ and $\logit(\fun) > 0$ as a function on $\points$, so $\logit(\fun)$ is a valid prox-strategy.
\endenv
\end{example}

%----------------------------------------------------------------------
%%% GENERAL
%----------------------------------------------------------------------
\section{General regret bounds}
\label{sec:general}
%----------------------------------------------------------------------
%%% GENERAL
%----------------------------------------------------------------------
% !TEX root = ./Main.tex

%----------------------------------------------------------------------
%%% STATIC REGRET
%----------------------------------------------------------------------
\subsection{Static regret guarantees}

We are now in a position to state our first result for \eqref{eq:DA}:

\begin{restatable}{proposition}{regsimple}
\label{prop:reg-simple}
For any simple strategy $\simple\in\simples$, \cref{alg:DA} enjoys the bound
\begin{equation}
\label{eq:reg-simple}
\txs
\reg_{\simple}(\nRuns)
	\leq \temp_{\nRuns+1}^{-1} \bracks{\hreg(\simple) - \min\hreg}
		+ \insum_{\run=\start}^{\nRuns} \braket{\error_{\run}}{\simple - \state_{\run}}
		+ \frac{1}{2\hstr} \insum_{\run=\start}^{\nRuns} \temp_{\run} \dnorm{\model_{\run}}^{2}.
\end{equation}
\end{restatable}

\cref{prop:reg-simple} is a ``template'' bound that we will use to extract static and dynamic regret guarantees in the sequel.
Its proof relies on the introduction of a suitable energy function measuring the match between the learner's aggregate model $\score_{\run}$ and the comparator $\simple$.
The main difficulty is that these variables live in completely different spaces ($\linf$ \vs $\simples$ respectively), so there is no clear distance metric connecting them.
However, since bounded functions $\dual\in\linf$ and simple strategies $\simple\in\simples$ are naturally paired via duality, they are indirectly connected via the Fenchel\textendash Young inequality $\braket{\dual}{\simple} \leq \hreg(\simple) + \hconj(\dual)$, where $\hconj(\dual) = \max_{\primal\in\vecspace}\{\braket{\dual}{\primal} - \hreg(\primal)\}$ denotes the convex conjugate of $\hreg$ and equality holds if and only if $\mirror(\dual) = \simple$.
%Thus, taking into account the inflation of the aggregate model $\score_{\run}$ by $\temp_{\run}$ in \eqref{eq:DA},
We will thus consider the energy function
\begin{equation}
\label{eq:energy}
\txs
\energy_{\run}
	\defeq \temp_{\run}^{-1}
		\bracks{\hreg(\simple)
			+ \hconj(\temp_{\run}\score_{\run})
			- \braket{\temp_{\run}\score_{\run}}{\simple}}.
\end{equation}
By construction, $\energy_{\run} \geq 0$ for all $\run$ and $\energy_{\run} = 0$ if and only if $\state_{\run} = \mirror(\temp_{\run}\score_{\run}) = \simple$.
More to the point, the defining property of $\energy_{\run}$ is the following recursive bound (which we prove in the supplement):

\begin{restatable}{lemma}{energybound}
\label{lem:energybound}
%The energy function $\energy_{\run}$ satisfies the inequality:
For all $\simple\in\simples$, we have:
\begin{equation}
\label{eq:energybound}
\txs
\energy_{\run+1}
	\leq \energy_{\run}
		+ \braket{\model_{\run}}{\simple - \state_{\run}}
		+ \parens*{\temp_{\run+1}^{-1} - \temp_{\run}^{-1}} \bracks{\hreg(\simple) - \min\hreg}
		+ \frac{\temp_{\run}}{2\hstr} \dnorm{\model_{\run}}^{2}.
\end{equation}
\end{restatable}

%Armed with this lemma,
\Cref{prop:reg-simple} is obtained by telescoping \eqref{eq:energybound};
%(also provided in the supplement).
subsequently, to obtain a regret bound for \cref{alg:DA}, we must relate $\reg_{\point}(\nRuns)$ to $\reg_{\simple}(\nRuns)$.
This can be achieved by invoking \cref{lem:point2simple}
but the resulting expressions are much simpler when $\hreg$ is \emph{decomposable}, \ie $\hreg(\primal) = \int_{\points} \hdec(\primal(\point)) \dd\point$ for some $C^{2}$ function $\hdec\from[0,\infty)\to\R$ with $\hdec''>0$.
%\footnote{For example, $\hdec(z) = z^{2}/2$ for the $L^{2}$ case (\cref{ex:L2}) and $\hdec(z) = z \log z$ for the entropy (\cref{ex:entropy}).}
%In view of this, we state our results in the decomposable setting:
In this more explicit setting, we have:

\begin{restatable}{theorem}{regstat}
\label{thm:reg-stat}
Fix $\point\in\points$, let $\cvx$ be a convex neighborhood of $\point$ in $\points$,
and
suppose that \cref{alg:DA} is run with a decomposable regularizer $\hreg(\primal) = \int_{\points} \hdec\circ\primal$.
Then, letting $\hvol(z) = z \hdec(1/z)$ for $z>0$, we have:
\begin{equation}
\label{eq:reg-bound-stat}
\exof{\reg_{\point}(\nRuns)}
	\leq \frac{\hvol(\leb(\cvx)) - \hvol(\leb(\points))}{\temp_{\nRuns+1}}
	+ \lips\diam(\cvx) \nRuns
	+ 2 \insum_{\run=\start}^{\nRuns} \bbound_{\run}
	+ \frac{\coef^{2}}{2\hstr} \insum_{\run=\start}^{\nRuns} \temp_{\run} \mbound_{\run}^{2}.
\end{equation}
%where $\cvx$ is a convex neighborhood of $\point$ in $\points$ and $\hvol(z) = z \hdec(1/z)$ for all $z>0$.
In particular, if
\cref{alg:DA} is run with learning rate $\temp_{\run} \propto 1/\run^{\pexp}$, $\pexp\in(0,1)$,
and
inexact models such that $\bbound_{\run} = \bigoh(1/\run^{\bexp})$ and $\mbound_{\run}^{2} = \bigoh(\run^{2\mexp})$ for some $\bexp,\mexp\geq0$,
we have:
\begin{equation}
\label{eq:reg-bound-stat-powers}
\exof{\reg(\nRuns)}
	= \bigoh(\hvol(\nRuns^{-\vdim\dexp}) \nRuns^{\pexp} + \nRuns^{1-\dexp} + \nRuns^{1-\bexp} + \nRuns^{1+2\mexp-\pexp})
	\quad
	\text{for all $\dexp\geq0$}.
\end{equation}
\end{restatable}
%\AH{It's not clear where the $\dexp$ comes from. I understand from line 228 that it comes from $\cvx$, but then we should note it $\cvx_T$, no ?}
%As an immediate consequence of this general bound, we obtain the following concrete guarantees:

\smallskip
\begin{corollary}
\label{cor:reg-stat}
If the learner's feedback is unbiased and bounded in mean square \textpar{\ie $\bbound_{\run}=0$ and $\sup_{\run} \mbound_{\run} < \infty$}, running \cref{alg:DA} with learning rate $\temp_{\run} \propto 1/\run^{\pexp}$ guarantees
\begin{equation}
\label{eq:reg-bound-stat-unbiased}
\exof{\reg(\nRuns)}
	= \bigoh(\hvol(\nRuns^{-\vdim\pexp})\nRuns^{\pexp} + \nRuns^{1-\pexp}).
\end{equation}
%\AH{I definitely should read carefully the proof. But why $\dexp$ becomes $\pexp$? I only understood that $\bexp=\infty$ and $\mexp =0$.}
%\PM{That's part of the calculation, nothing deep.}
In particular, for the regularizers of \cref{ex:L2,ex:entropy}, we have:
\begin{enumerate}
\item
For $\hdec(z) = (1/2) z^{2}$, \cref{alg:DA} with $\temp_{\run}\propto\run^{-1/(\vdim+2)}$ guarantees $\exof{\reg(\nRuns)} = \bigoh(\nRuns^{\frac{\vdim+1}{\vdim+2}})$.
\item
For $\hdec(z) = z\log z$, \cref{alg:DA} with $\temp_{\run}\propto\run^{-1/2}$ guarantees $\exof{\reg(\nRuns)} = \bigoh(\nRuns^{1/2})$.
\end{enumerate}
\end{corollary}

\begin{remark}
Here and in the sequel, logarithmic factors are ignored in the Landau $\bigoh(\cdot)$ notation.
We should also stress that the role of $\cvx$ in \cref{thm:reg-stat} only has to do with the analysis of the algorithm, not with the derived bounds (which are obtained by picking a suitable $\cvx$).
%The proof of \cref{thm:reg-stat,cor:reg-stat} is deferred to the supplement;
%We defer this discussion to the supplement;
%instead, we focus below on relating the derived bounds to existing results.
\end{remark}

%As evidenced by the above, the role of $\cvx$ in \cref{thm:reg-stat} only has to do with the analysis.
%The proof of \cref{thm:reg-stat,cor:reg-stat} is deferred to the supplement;
%we defer this discussion to the supplement and, instead, we focus here on relating these bounds to existing results in the literature.

First, in online \emph{convex} optimization, \acl{DA} with stochastic gradient feedback achieves $\bigoh(\sqrt{\nRuns})$ regret \emph{irrespective} of the choice of regularizer,
and this bound is tight \citep{ABRT08,SS11,BCB12}.
%Importantly, in convex-structured problems, this is achieved by \emph{any} admissible regularizer:
%the impact of the choice of $\hreg$ is limited to the multiplicative constant in the algorithm's regret bound;
%tuning it to obtain (almost) dimension-free bounds is an important research thread in the online learning literature \citep{BCB12}.
By contrast, in the non-convex setting, the choice of regularizer has a visible impact on the regret because it affects the exponent of $\nRuns$:
in particular, $L^{2}$ regularization carries a much worse dependence on $\nRuns$ relative to the Hedge variant of \cref{alg:DA}.
This is due to the term $\bigoh(\hvol(\nRuns^{-\vdim\dexp}) \nRuns^{\pexp})$ that appears in \eqref{eq:reg-bound-stat-powers} and is in turn linked to the choice of the ``enclosure set'' $\cvx$ having $\leb(\cvx) \propto \nRuns^{-\vdim\dexp}$ for some $\dexp\geq0$.

The negentropy regularizer (and any other regularizer with quasi-linear growth at infinity, see the supplement for additional examples) only incurs a logarithmic dependence on $\leb(\cvx)$.
Instead, the quadratic growth of the $L^{2}$ regularizer induces an $\bigoh(1/\leb(\cvx))$ term in the algorithm's regret, which is ultimately responsible for the catastrophic dependence on the dimension of $\points$.
%The above adds a further layer to the optimality claim of \cite{SN20} for FTPL with exponential perturbations (which also achieves an $\bigoh(\nRuns^{1/2})$ regret).
Seeing as the bounds achieved by the Hedge variant of \cref{alg:DA} are optimal in this regard, we will concentrate on this specific instance in the sequel.

%This adds a further layer to the optimality claim of \cite{SN20}.

%----------------------------------------------------------------------
%%% DYNAMIC REGRET
%----------------------------------------------------------------------
\subsection{Dynamic regret guarantees}

We now turn to the dynamic regret minimization guarantees of \cref{alg:DA}.
In this regard, we note first that,
in complete generality,
dynamic regret minimization is \emph{not possible} because an informed adversary can always impose a uniformly positive loss at each stage \citep{SS11}.
Because of this, dynamic regret guarantees are often stated in terms of the \emph{variation} of the loss functions encountered, namely
%Following \citet{BGZ15}, we define this variation as
\begin{equation}
\label{eq:tvar}
\txs
\tvar
	\defeq \insum_{\run=\start}^{\nRuns} \supnorm{\loss_{\run+1} - \loss_{\run}}
	= \insum_{\run=\start}^{\nRuns} \max_{\point\in\points} \abs{\loss_{\run+1}(\point) - \loss_{\run}(\point)},
\end{equation}
with the convention $\loss_{\run+1} = \loss_{\run}$ for $\run=\nRuns$.%
\footnote{This notion is due to \citet{BGZ15}.
Other notions of variation have also been considered \cite{CYLM+12,CBGLS12,BGZ15}, as well as other measures of regret, \cf \cite{HW98,HazSes09}; for a survey, see \cite{CBL06}.}
We then have:

\begin{restatable}{theorem}{regdyn}
\label{thm:reg-dyn}
Suppose that the Hedge variant of \cref{alg:DA} is run with learning rate $\temp_{\run} \propto 1/\run^{\pexp}$ and inexact models with $\bbound_{\run} = \bigoh(1/\run^{\bexp})$ and $\mbound_{\run}^{2} = \bigoh(\run^{2\mexp})$ for some $\bexp,\mexp\geq0$.
Then:
\begin{equation}
\label{eq:reg-bound-dyn}
\exof{\dynreg(\nRuns)}
	= \bigoh(\nRuns^{1+2\mexp-\pexp} + \nRuns^{1-\bexp} + \nRuns^{2\pexp-2\mexp}\tvar).
\end{equation}
In particular, if $\tvar = \bigoh(\nRuns^{\vexp})$ for some $\vexp<1$ and the learner's feedback is unbiased and bounded in mean square \textpar{\ie $\bbound_{\run}=0$ and $\sup_{\run} \mbound_{\run} < \infty$}, the choice $\pexp = (1-\vexp)/3$ guarantees
\begin{equation}
\label{eq:reg-bound-dyn-tuned}
\exof{\dynreg(\nRuns)}
	= \bigoh(\nRuns^{\frac{2+\vexp}{3}}).
\end{equation}
\end{restatable}
%\TR{wow}

To the best of our knowledge, \cref{thm:reg-dyn} provides the first dynamic regret guarantee for online non-convex problems.
The main idea behind its proof is to examine the evolution of play over a series of windows of length $\batch = \bigoh(\nRuns^{\qexp})$ for some $\qexp>0$.
In so doing, \cref{thm:reg-stat} can be used to obtain a bound for the learner's regret relative to the best action $\point\in\points$ \emph{within each window}.
Obviously, if the length of the window is chosen sufficiently small, aggregating the learner's regret \emph{per window} will be a reasonable approximation of the learner's dynamic regret.
At the same time, if the window is taken too small, the number of such windows required to cover $\nRuns$ will be $\Theta(\nRuns)$, so this approximation becomes meaningless.
As a result, to obtain a meaningful regret bound, this window-by-window examination of the algorithm must be carefully aligned with the variation $\tvar$ of the loss functions encountered by the learner.
Albeit intuitive, the details required to make this argument precise are fairly subtle, so we relegate the proof of \cref{thm:reg-dyn} to the paper's supplement.

We should also observe here that the $\bigoh(\nRuns^{\frac{2+\vexp}{3}})$ bound of \cref{thm:reg-dyn} is, in general, unimprovable, even if the losses are \emph{linear}.
Specifically, \citet{BGZ15} showed that, if the learner is facing a stream of linear losses with stochastic gradient feedback (\ie an inexact linear model), an informed adversary can still impose $\dynreg(\nRuns) = \Omega(\nRuns^{2/3}\tvar^{1/3})$.
\citet{BGZ15} further proposed a scheme to achieve this bound by means of a periodic restart meta-principle that partitions the horizon of play into batches of size $(\nRuns/\tvar)^{2/3}$ and then runs an algorithm achieving $(\nRuns/\tvar)^{1/3}$ regret per batch.
\cref{thm:reg-dyn} differs from the results of \citet{BGZ15} in two key aspects:
\begin{enumerate*}
[(\itshape a\hspace*{.5pt}\upshape)]
\item
\cref{alg:DA} does not require a periodic restart schedule (so the learner does not forget the information accrued up to a given stage);
and
\item
more importantly,
it applies to \emph{general} online optimization problems, without a convex structure or any other structural assumptions (though with a different feedback structure).
\end{enumerate*}

%----------------------------------------------------------------------
%%% BANDIT
%----------------------------------------------------------------------
\section{Applications to online non-convex learning with bandit feedback}
\label{sec:bandit}
%----------------------------------------------------------------------
%%% BANDIT
%----------------------------------------------------------------------
% !TEX root = ./Main.tex

As an application of the inexact model framework of the previous sections, we proceed to consider the case where the learner only observes their realized reward $\loss_{\run}(\choice_{\run})$ and has \emph{no other information}.
In this ``bandit setting'', an inexact model is not available and must instead be constructed on the fly.

When $\points$ is a finite set, $\loss_{\run}$ is a $\abs{\points}$-dimensional vector, and an unbiased estimator for $\loss_{\run}$ can be constructed by setting $\model_{\run}(\point) = [\oneof{\point = \choice_{\run}} / \probof{\point = \choice_{\run}}] \, \loss_{\run}(\choice_{\run})$ for all $\point\in\points$.
This ``importance weighted'' estimator is the basis for the EXP3 variant of the Hedge algorithm which is known to achieve $\bigoh(\nRuns^{1/2})$ regret \citep{ACBFS02}.
However, in the case of \emph{continuous} action spaces, there is a key obstacle:
if the indicator $\oneof{\point = \choice_{\run}}$ is replaced by a Dirac point mass $\delta_{\point_{\run}}(\point)$, the resulting loss model $\model_{\run} \propto \delta_{\point_{\run}}$ would no longer be a function but a generalized (singular) distribution, so the \acl{DA} framework of \cref{alg:DA} no longer applies.

To counter this, we will take a ``smoothing'' approach in the spirit of \cite{BLE17} and consider the estimator
%\PM{@Thibaud: do we need to change the estimator here?}
%\TR{Nope \smiley -- solved}
\begin{equation}
\label{eq:kernel}
\model_{\run}(\point)
	= \kerfun_{\run}(\choice_{\run},\point)
		\cdot \loss_{\run}(\choice_{\run}) / \state_{\run}(\choice_{\run})
\end{equation}
where
$\kerfun_{\run} \from \points\times\points \to \R$ is a (time-varying) \emph{smoothing kernel}, \ie $\int_{\points} \kerfun_{\run}(\point,\pointalt) \dd\pointalt = 1$ for all $\point\in\points$.
%[As before, $\choice_{\run}$ is the action chosen by the learner at stage $\run$ based on the probability distribution $\state_{\run}$.]
For concreteness (and sampling efficiency), we will assume that losses now take values in $[0,1]$, and we will focus on simple kernels that are supported on a neighborhood $\nhd_{\width}(\point) = \ball_{\width}(\point) \cap \points$ of $\point$ in $\points$ and are constant therein, \ie $\kerfun^{\width}(\point,\pointalt) = [\leb(\nhd_{\width}(\point))]^{-1} \oneof{\norm{\pointalt - \point} \leq \width}$.
%\footnote{Note that $\kerfun^{\width}$ is not symmetric;
%this is because $\points$ is a constrained set, so it is not invariant under reflections.}

The ``smoothing radius'' $\width$ in the definition of $\kerfun^{\width}$ will play a key role in the choice of loss model being fed to \cref{alg:DA}.
If $\width$ is taken too small, $\kerfun^{\width}$ will approach a point mass, so it will have low estimation error but very high variance;
at the other end of the spectrum, if $\width$ is taken too large, the variance of the induced estimator will be low, but so will its accuracy.
In view of this, we will consider a flexible smoothing schedule of the form $\width_{\run} = 1/\run^{\wexp}$ which gradually sharpens the estimator over time as more information comes in.
Then, to further protect the algorithm from getting stuck in local minima,
we will also incorporate in $\state_{\run}$ an explicit exploration term of the form $\mix_{\run}/\leb(\points)$.

%----------------------------------------------------------------------
%% BDA pseudocode begins here

\begin{algorithm}[tbp]
\ttfamily
\small
\caption{\Acl{BDA} \hfill [Hedge variant: $\mirror \subs \logit$]}
%----------------------------------------------------------------------
%%% BDA ALGORITHM
%----------------------------------------------------------------------
% !TEX root = ../Main.tex

\begin{algorithmic}[1]
\Require
	mirror map $\mirror\from\linf\to\proxdom$;
	parameters $\temp_{\run},\width_{\run},\mix_{\run} > 0$;
	\textbf{initialize:} $\score_{\start} \subs 0$
%\State
%	choose $\score_{\start} = 0$
%%	play $\State \subs \mirror(\score) = \argmin\hreg$
%	\Comment{initialization}%
\For{$\run=\running$}
	\State
		set $\state_{\run} \subs (1-\mix_{\run}) \mirror(\temp_{\run}\score_{\run}) + \mix_{\run} / \leb(\points)$
		\quad
		[$\mirror \subs \logit$ for Hedge]
		\Comment{mixed strategy}%
	\State
		play $\choice_{\run} \sim \state_{\run}$
		\Comment{choose action}%
	\State
		set $\model_{\run} = \kerfun^{\width_{\run}}(\choice_{\run},\cdot)
			\cdot \loss_{\run}(\choice_{\run}) / \state_{\run}(\choice_{\run})$
%		receive $\loss_{\run}(\choice_{\run})$
%		receive $\loss_{\run}(\choice_{\run})$ and observe $\model_{\run}$
		\Comment{payoff model}%
	\State
		set $\score_{\run+1} \subs \score_{\run} - \model_{\run}$
		\Comment{update scores}%
\EndFor
%\State
%	\Return $\State$
\end{algorithmic}
\label{alg:BDA}
\end{algorithm}

%% BDA pseudocode ends here
%----------------------------------------------------------------------

Putting all this together, we obtain the \acdef{BDA} algorithm presented in pseudocode form as \cref{alg:BDA} above.
By employing a slight variation of the analysis presented in \cref{sec:general} (basically amounting to a tighter bound in \cref{lem:energybound}), we obtain the following guarantees for \cref{alg:BDA}:
%Thanks to the general guarantees obtained for \acl{DA} with inexact models in \cref{sec:general}, the analysis of this new algorithm boils down to estimating the bias and variance of the smoothed importance weighted estimator \eqref{eq:kernel}; relegating the details to the supplement, we have:
%\TR{Maybe add smth like `Slightly more general versions of \cref{thm:reg-stat,thm:reg-dyn}, details of which are relegated to the supplement, yield the following guarantees for \cref{alg:BDA}' (then again, we are short in space and what is there is v. good)}
%\TR{Removed Lemma 3 here (since we cannot really mention $M_t^2$ in that context) -- the proof uses a finer majoration of the second-order moment / variance which holds only in the Hedge + IW-Kernelized case}
%\begin{lemma}
%\label{lem:kernel-stats}
%The model \eqref{eq:kernel} with $\kerfun_{\run}$ and $\state_{\run}$ as in \cref{alg:BDA} has $\bbound_{\run} = \bigoh(\width_{\run})$ and $\mbound_{\run}^{2} = \bigoh(1/(\mix_{\run}\width_{\run}^{\vdim}))$.
%\end{lemma}
% Armed with this lemma (which we prove in the supplement), an application of \cref{thm:reg-stat,thm:reg-dyn} then yields the following guarantees for \cref{alg:BDA}:
\begin{restatable}{proposition}{regstatbandits}
%\begin{proposition}[Static regret]
\label{prop:reg-stat-BDA}
Suppose that the Hedge variant of \cref{alg:BDA} is run with learning rate $\temp_{\run} \propto 1/\run^{\pexp}$ and smoothing/exploration schedules $\width_{\run} \propto 1/\run^{\wexp}$, $\mix_{\run} \propto 1/\run^{\bexp}$ respectively.
Then, the learner enjoys the bound
\begin{equation}
\label{eq:reg-stat-BDA}
\exof{\reg(\nRuns)}
= \bigoh(\nRuns^{\pexp} + \nRuns^{1-\wexp} + \nRuns^{1-\bexp} + \nRuns^{1+\vdim\wexp+\bexp-\pexp}).
\end{equation}
In particular, if the algorithm is run with
$\pexp=(\vdim+2)/(\vdim+3)$
and
$\mexp = \bexp = 1/(\vdim+3)$,
we obtain the bound
$\exof{\reg(\nRuns)} = \bigoh(\nRuns^{\frac{\vdim+2}{\vdim+3}})$.
%\end{proposition}
\end{restatable}

\begin{restatable}{proposition}{regdynbandits}
%\begin{proposition}[Dynamic regret]
\label{prop:reg-dyn-BDA}
Suppose that the Hedge variant of \cref{alg:BDA} is run with parameters as in \cref{prop:reg-stat-BDA} against a stream of loss functions with variation $\tvar = \bigoh(\nRuns^{\vexp})$.
Then, the learner enjoys
\begin{equation}
\label{eq:reg-dyn-BDA}
\exof{\dynreg(\nRuns)}
= \bigoh(\nRuns^{1+\vdim\wexp+\bexp-\pexp} + \nRuns^{1-\bexp} + \nRuns^{1-\wexp} + \nRuns^{\vexp+2\pexp-\vdim\wexp-\bexp}).
\end{equation}
In particular, if the algorithm is run with
$\pexp=(1-\vexp)(\vdim+2)/(\vdim+4)$
and
$\mexp = \bexp = (1-\vexp)/(\vdim+4)$,
we obtain the optimized bound
$\exof{\dynreg(\nRuns)} = \bigoh(\nRuns^{\frac{\vdim+3+\vexp}{\vdim+4}})$.
%\end{proposition}
\end{restatable}

To the best of our knowledge, \cref{prop:reg-dyn-BDA} is the first result of its kind for dynamic regret minimization in online non-convex problems with bandit feedback.
We conjecture that the bounds of \cref{prop:reg-stat-BDA,prop:reg-dyn-BDA} can be tightened further to $\bigoh(\nRuns^{\frac{\vdim+1}{\vdim+2}})$ and $\bigoh(\nRuns^{\frac{\vdim+2+\vexp}{\vdim+3}})$ by dropping the explicit exploration term;
we defer this finetuning to future work.

%----------------------------------------------------------------------
%%% NUMERICS
%----------------------------------------------------------------------
%\section{Numerical experiments}
%\label{sec:numerics}
%\input{Numerics}

%----------------------------------------------------------------------
%%% CONCLUSION
%----------------------------------------------------------------------
%\section{Concluding remarks}
%\label{sec:conclusion}
%\input{Conclusion}

%*************************************************************
%*****    APPENDIX
%*************************************************************
\appendix
\numberwithin{equation}{section}		% for numbering  in the appendix
\numberwithin{lemma}{section}		% for numbering  in the appendix
\numberwithin{proposition}{section}		% for numbering  in the appendix
\numberwithin{theorem}{section}		% for numbering in the appendix

%----------------------------------------------------------------------
%%% APP: EXAMPLES
%----------------------------------------------------------------------
\section{Examples}
\label{app:examples}
%----------------------------------------------------------------------
%%% APP: EXAMPLES
%----------------------------------------------------------------------
% !TEX root = ./Main.tex

In this appendix, we provide some more decomposable regularizers that are commonly used in the literature:

\begin{example}[Log-barrier regularization]
\label{ex:barrier}
Let $\vecspace = \measures(\points)$ as above and consider the so-called \emph{Burg entropy} $\hreg(\primal) = -\int_{\points} \log\primal$ \citep{ABB04}.
In this case, $\proxdom = \dom\subd\hreg = \setdef{\proxal\in\strats_{\abscont}}{\supp(\proxal) = \points} = \dom\hreg$ and $\nabla\hreg(\proxal) = -1/\proxal$ on $\proxdom$.
In the finite-dimensional case, this regularizer plays a fundamental role in the affine scaling method of \citet{Kar90}, see \eg \citet{Tse04,VMF86} and references therein.
The corresponding mirror map is obtained as follows:
let $L(\primal;\lambda) = \int_{\points} \dual \primal + \int_{\points} \log\primal - \lambda \int_{\points} \primal$ denote the Lagrangian of the problem \eqref{eq:mirror}, so $\proxal = \mirror(\dual)$ satisfies the first-order optimality condition
\begin{equation}
\label{eq:KKT-barrier}
\dual + 1/\proxal - \lambda = 0.
\end{equation}
Solving for $\proxal$ and integrating, we get $\int_{\points} (\lambda - \dual)^{-1} = \int_{\points}\proxal = 1$.
The function $\phi(\lambda) = \int_{\points} (\lambda - \dual)^{-1}$ is decreasing in $\lambda$ and continuous whenever finite;
moreover, since $\dual\in\linf$, it follows that $\phi$ is always finite (and hence continuous) for large enough $\lambda$, and $\lim_{\lambda\to\infty} \phi(\lambda) = 0$.
Since $\sup_{\lambda} \phi(\lambda) = \infty$, there exists some maximal $\lambda^{*}$ such that \eqref{eq:KKT-barrier} holds (in practice, this can be located by a simple line search initialized at some $\lambda > \supnorm{\psi}$).
We thus get $\mirror(\dual) = (\lambda^{\ast} - \dual)^{-1}$.
\endenv
\end{example}

\begin{example}[Tsallis entropy]
\label{ex:Tsallis}
A generalization of the Shannon-Gibbs entropy for nonextensive variables is the \emph{Tsallis entropy} \citep{Tsa88} defined here as $\hreg(\primal) = \int_{\points} \hdec(\primal)$ where $\hdec(z) = [\gamma(1-\gamma)]^{-1} (z - z^{\gamma})$ for $\gamma\in(0,1]$, with the continuity convention $(z-z^{\gamma}) / (1-\gamma) = z\log z$ for $\gamma=1$ (corresponding to the Shannon-Gibbs case).
Working as in \cref{ex:barrier}, we have $\proxdom = \dom\subd\hreg = \setdef{\proxal\in\strats_{\abscont}}{\supp(\proxal) = \points} \subsetneq \dom\hreg$, and the corresponding mirror map $\proxal = \mirror(\dual)$ is obtained via the first-order stationarity equation
\begin{equation}
\label{eq:KKT-Tsallis}
\dual - \frac{1-\gamma \proxal^{\gamma-1}}{\gamma(1-\gamma)} - \lambda
	= 0.
\end{equation}
Then, solving for $\proxal$ yields $\mirror(\dual) = (1-\gamma)^{1/(\gamma-1)} \int_{\points} (\mu - \dual)^{1/(\gamma-1)}$ with $\mu > \supnorm{\dual}$ chosen so that $\int_{\points} \mirror(\dual) = 1$.
\endenv
\end{example}

%----------------------------------------------------------------------
%%% APP: MIRROR
%----------------------------------------------------------------------
\section{Basic properties of regularizers and mirror maps}
\label{app:mirror}
%----------------------------------------------------------------------
%%% APP: MIRROR
%----------------------------------------------------------------------
% !TEX root = ./Main.tex

The goal of this appendix is to prove some basic results on regularizer functions and mirror maps that will be used liberally in the sequel.
Versions of the results presented here already exist in the literature, but our infinite-dimensional setting introduces some subtleties that require further care.
For this reason, we state and prove all required results for completeness.

We begin by recalling some definitions from the main part of the paper.
%First, $\dspace$ will denote the space of continuous functions on $\vecspace$ endowed with the supremum norm $\supnorm{\pay} = \max_{\point\in\vecspace} \abs{\pay(\point)}$.
First, we write $\measures \equiv \measures(\points)$ for the space of all finite signed Radon measures on $\points$ equipped with the total variation norm $\tvnorm{\meas} = \meas^{+}(\points) + \meas^{-}(\points)$, where $\meas^{+}$ (resp. $\meas^{-}$) denotes the positive (resp. negative) part of $\meas$ coming from the Hahn-Banach decomposition of signed measures on $\points$.
As we discussed in \cref{sec:DA}, we also assume given a model Banach space $\vecspace$ containing the set of simple strategies $\simples$ as an embedded subset and such that $\tvnorm{\cdot} \leq \coef \norm{\cdot}$ for some $\coef>0$.

With all this in hand, we begin by discussing the well-posedness of \cref{alg:DA}.
To that end, we have the following basic result:

\begin{lemma}
\label{lem:mirror}
Let $\hreg$ be a regularizer on $\simples$.
Then:
\begin{enumerate}
\item
$\mirror(\dual) \in \proxdom$ for all $\dual\in\dspace$; in particular:
\begin{equation}
\label{eq:links-mirror}
\proxal = \mirror(\dual)
	\;\iff\;
	\dual \in \subd\hreg(\proxal).
\end{equation}

%\item
%For all $\dvec \in \dspace$, we have:
%\begin{equation}
%\label{eq:links-prox}
%\new\proxal = \mirror(\nabla\hreg(\proxal) + \dvec)
%	\;\iff\;
%\nabla\hreg(\proxal) + \dvec \in \subd\hreg(\new\proxal)
%\end{equation}

\item
If $\proxal = \mirror(\dual)$ and $\primal\in\dom\hreg$, we have
\begin{equation}
\label{eq:selection}
\braket{\nabla\hreg(\proxal)}{\proxal - \primal}
	\leq \braket{\dual}{\proxal - \primal}.
\end{equation}

\item
The convex conjugate $\hconj(\dual) = \max_{\primal\in\vecspace} \{ \braket{\dual}{\primal} - \hreg(\primal) \}$ is Fréchet differentiable and satisfies
\begin{equation}
\label{eq:dconj}
\dif_{\dvec}\hconj(\dual)
	= \braket{\dvec}{\mirror(\dual)}
	\quad
	\text{for all $\dual,\dvec\in\dspace$}.
\end{equation}
\end{enumerate}
\end{lemma}

\begin{corollary}
\cref{alg:DA} is well-posed, \ie $\state_{\run} \in \proxdom$ for all $\run=\running$ if $\model_{\run}\in\linf$.
\end{corollary}

\begin{proof}
We proceed item by item:
\begin{enumerate}[leftmargin=*]
\item
First, since $\hreg$ is strongly convex and \acl{lsc}, the maximum in \eqref{eq:mirror} is attained.
Hence, by Fermat's rule for subdifferentials, $\proxal$ solves \eqref{eq:mirror} if and only if $\dual - \subd\hreg(\proxal) \ni 0$.
We thus get the string of equivalences:
\begin{equation}
\proxal = \mirror(\dual)
	\iff
	\dual - \subd\hreg(\proxal) \ni 0
	\iff
	\dual \in \subd\hreg(\proxal).
\end{equation}
In particular, this implies that $\dual \in \dom\subd\hreg(\proxal) \neq \varnothing$, \ie $\proxal\in\dom\subd\hreg \eqdef \proxdom$, as claimed.

\item
To establish \eqref{eq:selection}, it suffices to show that it holds for all $\proxal\in\proxdom$ (by continuity).
To do so, let
\begin{equation}
\phi(t)
	= \hreg(\proxal + t(\primal-\proxal))
	- \bracks{\hreg(\proxal) +  \braket{\dual}{\proxal + t(\primal-\proxal)}}.
\end{equation}
Since $\hreg$ is strongly convex relative $\dual\in\subd\hreg(\proxal)$ by \eqref{eq:links-mirror}, it follows that $\phi(t)\geq0$ with equality if and only if $t=0$.
Moreover, note that $\psi(t) = \braket{\nabla\hreg(\proxal + t(\primal-\proxal)) - \dual}{\primal - \proxal}$ is a continuous selection of subgradients of $\phi$.
Given that $\phi$ and $\psi$ are both continuous on $[0,1]$, it follows that $\phi$ is continuously differentiable and $\phi' = \psi$ on $[0,1]$.
Thus, with $\phi$ convex and $\phi(t) \geq 0 = \phi(0)$ for all $t\in[0,1]$, we conclude that $\phi'(0) = \braket{\nabla\hreg(\proxal) - \dual}{\primal - \proxal} \geq 0$, from which our claim follows.
\end{enumerate}
Finally, the Fréchet differentiability of $\hconj$ is a straightforward application of the envelope theorem, which is sometimes referred to in the literature as Danskin's theorem, \cf \citet[Chap.~4]{Ber97} 
\end{proof}

As we mentioned in the main text, much of our analysis revolves around the energy function \eqref{eq:energy} defined by means of the Fenchel-Young inequality.
To formalize this, it will be convenient to introduce a more general pairing between $\primal \in \vecspace$ and $\dual \in \dspace$, known as the \emph{Fenchel coupling}.
Following \citep{MZ19}, this is defined as
\begin{equation}
\label{eq:Fench}
\fench(\primal,\dual)=\hreg(\primal)+\hconj(\dual)-\braket{\dual}{\primal}
	\quad
	\text{for all $\primal \in \dom\hreg,\dual \in \dspace$}.
\end{equation}

The following series of lemmas gathers some basic properties of the Fenchel coupling.
The first is a lower bound for the Fenchel coupling in terms of the ambient norm in $\vecspace$:

\begin{lemma}
\label{lem:Fenchel-strong}
Let $\hreg$ be a regularizer on $\simples$ with strong convexity modulus $\hstr$.
Then, for all $\primal \in \dom\hreg$ and all $\dual \in \dspace$, we have
\begin{equation}
\label{eq:Fenchel-strong}
\fench(\primal,\dual)
	\geq \frac{\hstr}{2} \norm{\mirror(\dual) - \primal}^{2}.
\end{equation}
%\begin{enumerate}
%\item
%$\fench(\primal,\dual) = \breg(\primal,\mirror(\dual))$
%	if $\mirror(\dual) \in \subpoints$ (but not necessarily otherwise).
%\item
%If $\point = \mirror(\dual)$,
%	then $\fench(\primal,\dual) \geq \frac{\hstr}{2} \norm{\point-\primal}^{2}_{\point}$
%\end{enumerate} 
\end{lemma}

\begin{proof}
By the definition of $\fench$ and the inequality \eqref{eq:selection}, we have:
\begin{align}
\fench(\primal,\dual)
	&= \hreg(\primal)
		+ \hconj(\dual)
		- \braket{\dual}{\primal}
	= \hreg(\primal)
		+ \braket{\dual}{\mirror(\dual)}
		- \hreg(\mirror(\dual))
		- \braket{\dual}{\primal}
	\notag\\
	&\geq \hreg(\primal)
		- \hreg(\mirror(\dual))
		- \braket{\nabla\hreg(\dual)}{\mirror(\dual) - \primal}
	\notag\\
	&\geq \frac{\hstr}{2} \norm{\mirror(\dual) - \primal}^{2}
\end{align}
where
we used \eqref{eq:selection} in the second line,
and the strong convexity of $\hreg$ in the last.
\end{proof}

Our next result is the primal-dual analogue of the so-called ``three-point identity'' for the Bregman divergence \citep{CT93}:

\begin{proposition}
\label{prop:3points}
Let $\hreg$ be a regularizer on $\simples$,
fix some $\primal \in \vecspace$, $\dual,\new\dual \in \dspace$, and let $\proxal = \mirror(\dual)$.
Then:
\begin{equation}
\label{eq:3points}
\fench(\primal,\new\dual)
	= \fench(\primal,\dual)
	+ \fench(\proxal,\new\dual)
	+ \braket{\new\dual-\dual}{\proxal - \primal}.
\end{equation}
\end{proposition}

\begin{proof}
By definition:
\begin{equation}
\begin{aligned}
\fench(\primal,\new\dual)
	&= \hreg(\primal) +\hconj(\new\dual) - \braket{\new\dual}{\primal}
	\\
\fench(\primal,\dual)\hphantom{'}
	&= \hreg(\primal) + \hconj(\dual) - \braket{\dual}{\primal}.
	\\
\end{aligned}
\end{equation}
Thus, by subtracting the above, we get:
\begin{align}
\fench(\primal,\new\dual)-\fench(\primal,\dual)
	&= \hreg(\primal)
		+ \hconj(\new\dual)
		- \braket{\new\dual}{\primal}
		- \hreg(\primal)
		- \hconj(\dual)
		+ \braket{\dual}{\primal}
	\notag\\
	&=\hconj(\new\dual)
		- \hconj(\dual)
		- \braket{\new\dual - \dual}{\primal}
	\notag\\
	&=\hconj(\new\dual)
		- \braket{\dual}{\mirror(\dual)}
		+ \hreg(\mirror(\dual))
		- \braket{\new\dual - \dual}{\primal}
	\notag\\
	&=\hconj(\new\dual)
		- \braket{\dual}{\proxal}
		+ \hreg(\proxal)
		- \braket{\new\dual - \dual}{\primal}
	\notag\\
	&= \hconj(\new\dual)
		+ \braket{\new\dual - \dual}{\proxal}
		- \braket{\new\dual}{\proxal}
		+ \hreg(\proxal)
		- \braket{\new\dual - \dual}{\primal}
	\notag\\
	&= \fench(\proxal,\new\dual)
		+ \braket{\new\dual - \dual}{\proxal - \primal}
\end{align}
and our proof is complete.
\end{proof}

We are now in a position to state and prove a key inequality for the Fenchel coupling:

\begin{proposition}
\label{prop:Fench-bound}
Let $\hreg$ be a regularizer on $\simples$ with convexity modulus $\hstr$,
fix some $\primal\in\dom\hreg$,
and
let $\proxal = \mirror(\dual)$ for some $\dual\in\dspace$.
%and fix
%$\gvec\in\dspace$ and $\step>0$. 
Then, for all $\dvec\in\dspace$, we have:
\begin{equation}
\label{eq:Fench-bound}
\fench(\primal,\dual + \dvec)
	\leq \fench(\primal,\dual)
	+ \braket{\dvec}{\proxal - \primal}
	+ \frac{1}{2\hstr}\dnorm{\dvec}^{2}
\end{equation}
\end{proposition}

\begin{proof}
Let $\proxal = \mirror(\dual)$, $\new\dual = \dual + \dvec$, and $\new\proxal = \mirror(\new\dual)$.
Then, by the three-point identity \eqref{eq:3points}, we have
\begin{equation}
\fench(\primal,\dual)
	= \fench(\primal,\new\dual)
		+ \fench(\new\proxal,\dual)
		+ \braket{\dual - \new\dual}{\new\proxal - \primal}.
\end{equation}
Hence, after rearranging:
\begin{align}
\label{eq:Fenchel-prebound}
\fench(\primal,\new\dual)
	&= \fench(\primal,\dual)
		- \fench(\new\proxal,\dual)
		+ \braket{\dvec}{\new\proxal - \primal}
	\notag\\
	&=\fench(\primal,\dual)
		- \fench(\new\proxal,\dual)
		+ \braket{\dvec}{\proxal - \primal}
		+ \braket{\dvec}{\new\proxal - \proxal}.
\end{align} 
%where, in the last step, we used \eqref{eq:selection} and the fact that $\new\proxal = \mirror(\nabla\hreg(\proxal) + \dvec)$.
%[Recall here that $\new\proxal = \prox_{\proxal}(\dvec)$, so \eqref{eq:links-prox2} applies.]
%The above is just \eqref{eq:Bregman-old2new}, so the first part of our proof is complete.
By Young's inequality, we also have
\begin{equation}
\label{eq:Young}
\braket{\dvec}{\new\proxal - \proxal}
	\leq \frac{\hstr}{2} \norm{\new\proxal - \proxal}^{2}
	+ \frac{1}{2\hstr} \dnorm{\dvec}^{2}.
\end{equation}
Thus, substituting in \eqref{eq:Fenchel-prebound}, we get
\begin{equation}
\fench(\primal,\new\dual)
	\leq \fench(\primal,\dual)
		+ \braket{\dvec}{\proxal - \primal}
		+ \frac{1}{2\hstr} \dnorm{\dvec}^{2}
		- \fench(\new\proxal,\dual)
		+ \frac{\hstr}{2} \norm{\new\proxal - \proxal}^{2}.
\end{equation}
Our claim then follows by noting that $\fench(\new\proxal,\dual) \geq \frac{\hstr}{2} \norm{\new\proxal - \proxal}^{2}$ (\cf \cref{lem:Fenchel-strong} above).
\end{proof}

%----------------------------------------------------------------------
%%% APP: REGRET
%----------------------------------------------------------------------
\section{Regret derivations}
\label{app:regret}
%----------------------------------------------------------------------
%%% APP: REGRET
%----------------------------------------------------------------------
% !TEX root = ./Main.tex

\para{Notation: from losses to payoffs}
In this appendix, we prove the general regret guarantees for \cref{alg:DA}.
For notational convenience, we will switch in what follows from ``losses'' to ``payoffs'', \ie we will assume that the learner is encountering a sequence of payoff functions $\pay_{\run} = -\loss_{\run}$ and gets as feedback the model $\est\pay_{\run} = -\model_{\run}$.
\renewcommand{\model}{\pay}

%----------------------------------------------------------------------
%%% BASIC
%----------------------------------------------------------------------
\subsection{Basic bounds and preliminaries}

We begin by providing some template regret bounds that we will use as a toolkit in the sequel.
As a warm-up, we prove the basic comparison lemma between simple and pure strategies:

\pointsimple*

\begin{proof}
By \cref{asm:loss}, we have $\pay_{\run}(\point) \leq \pay_{\run}(\pointalt) + \lips \norm{\point - \pointalt} \leq \pay_{\run}(\pointalt) + \lips \diam(\nhd)$ for all $\pointalt\in\nhd$.
Hence, taking expectations on both sides relative to $\simple$, we get $\pay_{\run}(\point) \leq \braket{\pay_{\run}}{\simple} + \lips\diam(\nhd)$.
Our claim then follows by summing over $\run=\running,\nRuns$ and invoking the definition of the regret.
\end{proof}

We now turn to the derivation of our main regret guarantees as outlined in \cref{sec:general}.
Much of the analysis to follow will revolve around the energy function \eqref{eq:energy} which, for convenience, we restate below in terms of the Fenchel coupling \eqref{eq:Fench}:
\begin{equation}
%\label{eq:energy}
\tag{\ref*{eq:energy}}
\energy_{\run}
	\defeq \frac{1}{\temp_{\run}}
		\bracks{\hreg(\simple)
			+ \hconj(\temp_{\run}\score_{\run})
			- \braket{\temp_{\run}\score_{\run}}{\simple}}
	= \frac{1}{\temp_{\run}} \fench(\simple,\temp_{\run}\score_{\run}).
\end{equation}
In words, $\energy_{\run}$ essentially measures the primal-dual ``distance'' between the benchmark strategy $\simple$ and the aggregate model $\score_{\run}$, taking into account the inflation of the latter by $\temp_{\run}$ in \eqref{eq:DA}.
Our overall proof strategy will then be to relate the regret incurred by the optimizer to the evolution of $\energy_{\run}$ over time.
%To that end, it will be convenient to assume momentarily that $\min\hreg=0$;
%this can always be achieved via the transformation $\hreg \subs \hreg - \min\hreg$, so there is no loss of generality.
To that end, an application of Abel's summation formula gives:
\begin{subequations}
\label{eq:energy-basic}
\begin{align}
\energy_{\run+1} - \energy_{\run}
	&= \frac{1}{\temp_{\run+1}} \fench(\simple,\temp_{\run+1}\score_{\run+1})
		- \frac{1}{\temp_{\run}} \fench(\simple,\temp_{\run}\score_{\run})
	\notag\\
	&\label{eq:energy-const}
	= \frac{1}{\temp_{\run+1}} \fench(\simple,\temp_{\run+1}\score_{\run+1})
		- \frac{1}{\temp_{\run}} \fench(\simple,\temp_{\run}\score_{\run+1})
	\\
	&\label{eq:energy-update}
	+ \frac{1}{\temp_{\run}} \fench(\simple,\temp_{\run}\score_{\run+1})
		- \frac{1}{\temp_{\run}} \fench(\simple,\temp_{\run}\score_{\run}).
\end{align}
\end{subequations}
We now proceed to unpack the two terms \eqref{eq:energy-const} and \eqref{eq:energy-update} separately, beginning with the latter.

To do so, substituting $\primal \subs \simple$, $\dual \subs \temp_{\run}\score_{\run}$ and $\new\dual \subs \temp_{\run}\score_{\run+1}$ in \cref{prop:3points} yields
\begin{align}
\eqref{eq:energy-update}
	&= \frac{1}{\temp_{\run}} \bracks{\fench(\simple,\temp_{\run}\score_{\run} + \temp_{\run}\model_{\run}) - \fench(\simple,\temp_{\run}\score_{\run})}
	\notag\\
	&= \frac{1}{\temp_{\run}} \bracks{\fench(\state_{\run},\temp_{\run}\score_{\run+1}) + \braket{\temp_{\run}\model_{\run}}{\state_{\run} - \simple}}
	\notag\\
	&= \frac{\fench(\state_{\run},\temp_{\run}\score_{\run+1})}{\temp_{\run}}
		+ \braket{\model_{\run}}{\state_{\run} - \simple}
\end{align}
where we used the definition $\state_{\run} = \mirror(\temp_{\run}\score_{\run})$ of $\state_{\run}$.
We thus obtain the interim expression
\begin{align}
\label{eq:energy-inter}
\energy_{\run+1}
	&= \energy_{\run}
		+ \eqref{eq:energy-const}
		+ \braket{\model_{\run}}{\state_{\run} - \simple}
		+ \frac{\fench(\state_{\run},\temp_{\run}\score_{\run+1})}{\temp_{\run}}
%		+ \frac{1}{\temp_{\run+1}} \fench(\simple,\temp_{\run+1}\score_{\run+1})
%		- \frac{1}{\temp_{\run}} \fench(\simple,\temp_{\run}\score_{\run+1})
\end{align}

Moving forward, for the term \eqref{eq:energy-const}, the definition of the Fenchel coupling \eqref{eq:Fench} readily yields:
\begin{align}
\eqref{eq:energy-const}
	&= \bracks*{\frac{1}{\temp_{\run+1}} - \frac{1}{\temp_{\run}}} \hreg(\simple)
		+ \frac{1}{\temp_{\run+1}} \hconj(\temp_{\run+1}\score_{\run+1}) - \frac{1}{\temp_{\run}} \hconj(\temp_{\run}\score_{\run+1}).
\end{align}
Consider now the function $\varphi(\temp) = \temp^{-1} [\hconj(\temp\dual) + \min\hreg]$ for arbitrary $\dual\in\linf$.
By \cref{lem:mirror}, $\hconj$ is Fréchet differentiable with $\dif_{\dvec}\hconj(\cdot) = \braket{\dvec}{\mirror(\cdot)}$ for all $\dvec\in\dspace$, so a simple differentiation yields
\begin{align}
\varphi'(\temp)
	&= \frac{1}{\temp} \braket{\dual}{\mirror(\temp\dual)}
		-\frac{1}{\temp^{2}} \bracks{ \hconj(\temp\dual) + \min\hreg }
	\notag\\
	&= \frac{1}{\temp^{2}} \bracks{ \braket{\temp\dual}{\mirror(\temp\dual)} - \hconj(\temp\dual) - \min\hreg}
	\notag\\
	&= \frac{1}{\temp^{2}} \bracks{\hreg(\mirror(\temp\dual)) - \min\hreg}
	\geq 0,
\end{align}
where
%$\proxal = \mirror(\temp\dual)$ and
we used the Fenchel-Young inequality as an equality in the second-to-last line.
%and the fact that $\min\hreg = 0$ in the last one.
Since $\temp_{\run+1} \leq \temp_{\run}$, the above shows that $\varphi(\temp_{\run}) \geq \varphi(\temp_{\run+1})$.
Hence, substituting $\dual \subs \score_{\run+1}$, we ultimately obtain
\begin{equation}
\label{eq:hconj-delta}
\frac{1}{\temp_{\run+1}} \hconj(\temp_{\run+1}\score_{\run+1})
	- \frac{1}{\temp_{\run}} \hconj(\temp_{\run}\score_{\run+1})
	\leq \bracks*{\frac{1}{\temp_{\run}} - \frac{1}{\temp_{\run+1}}} \min\hreg.
\end{equation}

Therefore, combining \eqref{eq:energy-inter} and \eqref{eq:hconj-delta}, we have proved the following template bound:

\begin{lemma}
\label{lem:template}
For all $\simple\in\simples$, the policy \eqref{eq:DA} enjoys the bound
\begin{equation}
\label{eq:template}
\energy_{\run+1}
	\leq \energy_{\run}
		+ \braket{\model_{\run}}{\state_{\run} - \simple}
		+ \parens*{\frac{1}{\temp_{\run+1}} - \frac{1}{\temp_{\run}}} \bracks{\hreg(\simple) - \min\hreg}
		+ \frac{1}{\temp_{\run}} \fench(\state_{\run},\temp_{\run}\score_{\run+1}).
\end{equation}
\end{lemma}

We are now in a position to prove our basic energy inequality (restated below for convenience):

\energybound*

\begin{proof}
Going back to \cref{prop:Fench-bound} and setting $\primal \subs \state_{\run}$, $\dual \subs \temp_{\run}\score_{\run}$ and $\dvec \subs \temp_{\run}\model_{\run}$, we get
\begin{equation}
\fench(\state_{\run},\temp_{\run}\score_{\run+1})
	\leq \fench(\state_{\run},\temp_{\run}\score_{\run})
		+ \braket{\temp_{\run}\model_{\run}}{\state_{\run} - \state_{\run}}
		+ \frac{\temp_{\run}^{2}}{2\hstr} \dnorm{\model_{\run}}^{2}
	= \frac{\temp_{\run}^{2}}{2\hstr} \dnorm{\model_{\run}}^{2}
\end{equation}
where we used the fact that $\state_{\run} = \mirror(\temp_{\run}\score_{\run})$.
Our claim then follows by dividing both sides by $\temp_{\run}$ and substituting in \cref{lem:template}.
\end{proof}

We will come back to these results as needed.

%----------------------------------------------------------------------
%%% STATIC
%----------------------------------------------------------------------
\subsection{Static regret guarantees}

We are now ready to prove our static regret results for \cref{alg:DA}.
We begin with the precursor to our main result in that respect:

\regsimple*

\begin{proof}
Recalling the decomposition $\model_{\run} = \pay_{\run} + \error_{\run}$ for the learner's inexact models, a simple rearrangement of \cref{lem:energybound} gives
\begin{equation}
\braket{\pay_{\run}}{\simple - \state_{\run}}
	\leq \energy_{\run} - \energy_{\run+1}
		+ \braket{\error_{\run}}{\state_{\run} - \simple}
		+ \parens*{\temp_{\run+1}^{-1} - \temp_{\run}^{-1}} \bracks{\hreg(\simple) - \min\hreg}
		+ \frac{\temp_{\run}}{2\hstr} \dnorm{\model_{\run}}^{2}.
\end{equation}
Thus, telescoping over $\run=\running,\nRuns$, we get
\begin{align}
\reg_{\simple}(\nRuns)
	&\leq \energy_{\start} - \energy_{\nRuns+1}
		+ \parens*{\frac{1}{\temp_{\nRuns+1}} - \frac{1}{\temp_{\start}}} \bracks{\hreg(\simple) - \min\hreg}
		+ \sum_{\run=\start}^{\nRuns} \braket{\error_{\run}}{\state_{\run} - \simple}
		+ \frac{1}{2\hstr} \sum_{\run=\start}^{\nRuns} \temp_{\run} \dnorm{\model_{\run}}^{2}
	\notag\\
	&\leq \frac{\hreg(\simple) - \min\hreg}{\temp_{\nRuns+1}}
		+ \sum_{\run=\start}^{\nRuns} \braket{\error_{\run}}{\state_{\run} - \simple}
		+ \frac{1}{2\hstr} \sum_{\run=\start}^{\nRuns} \temp_{\run} \dnorm{\model_{\run}}^{2},
\end{align}
where we used the fact that
$\energy_{\run} \geq 0$ for all $\run$
and
$\energy_{\start} = \temp_{\start}^{-1} \bracks{\hreg(\simple) + \hconj(0)} = \temp_{\start}^{-1} \bracks{\hreg(\simple) - \min\hreg}$.
\end{proof}

As a simple application of \cref{lem:energybound}, we get the following bound for simple comparators:

\begin{corollary}
\label{cor:reg-avg-simple}
For all $\simple\in\simples$, \cref{alg:DA} guarantees
\begin{equation}
\label{eq:reg-stat-simple}
\exof{\reg_{\simple}(\nRuns)}
	\leq \frac{\hreg(\simple) - \min\hreg}{\temp_{\nRuns+1}}
		+ 2\sum_{\run=\start}^{\nRuns} \bbound_{\run}
		+ \frac{1}{2\hstr} \sum_{\run=\start}^{\nRuns} \temp_{\run} \exof{\dnorm{\model_{\run}}^{2}},
\end{equation}
\end{corollary}

\begin{proof}
Simply take expectations over \eqref{eq:reg-simple} and use the fact that
\begin{equation*}
\exof{\braket{\error_{\run}}{\state_{\run} - \simple}}
	= \exof{ \braket{\exof{\error_{\run} \given \filter_{\run}}}{\state_{\run} - \simple} }
	= \exof{ \braket{\bias_{\run}}{\state_{\run} - \simple} \given }
	\leq \exof{ \supnorm{\bias_{\run}} \onenorm{\state_{\run} - \simple}}
	\leq 2 \bbound_{\run}.
\qedhere
\end{equation*}
\end{proof}

We are finally in a position to prove the main static regret guarantee of \cref{alg:DA}:

\regstat*

\begin{proof}
To simplify the proof, we will make the normalizing assumption $\theta(0) = 0$;
if this is not the case, $\theta$ can always be shifted by $\theta(0)$ for this condition to hold.
[Note that \cref{ex:L2,ex:entropy} both satisfy this convention.]

With this in mind, let $\cvx$ be a convex neighborhood of $\point$ in $\points$, and let $\unif_{\cvx} = \leb(\cvx)^{-1} \one_{\cvx}$ denote the (simple) strategy that assigns uniform probability to the elements of $\cvx$ and zero to all other points in $\cvx$.
We then have:
\begin{equation}
\hreg(\unif_{\cvx})
	= \int_{\points} \hdec(\unif_{\cvx})
	= \int_{\points} \hdec(\one_{\cvx} / \leb(\cvx))
	= \int_{\cvx} \hdec(1/\leb(\cvx))
	= \leb(\cvx) \hdec(1/\leb(\cvx))
	= \hvol(\leb(\cvx)).
\end{equation}
Moreover, since $\hreg$ is decomposable and the probability constraint $\int_{\points} \simple = 1$ is symmetric, the minimum of $\hreg$ over $\simples$ will be attained at the uniform strategy $\unif_{\points} = \leb(\points)^{-1} \one_{\points}$.
Thus, with $\simples$ weakly dense in $\dom\hreg$, we obtain
\begin{equation}
\min\hreg
	= \hreg(\unif_{\points})
	= \int_{\points} \hdec(\one_{\points} / \leb(\points))
%	= \leb(\cvx) \hdec(1/\leb(\cvx))
	= \hvol(\leb(\points)).
\end{equation}
In view of all this, \cref{cor:reg-avg-simple} applied to $\simple=\unif_{\cvx}$ yields
\begin{equation}
\exof{\reg_{\simple}(\nRuns)}
	\leq \frac{\hvol(\leb(\cvx)) - \hvol(\leb(\points))}{\temp_{\nRuns+1}}
		+ 2\sum_{\run=\start}^{\nRuns} \bbound_{\run}
		+ \frac{\coef^{2}}{2\hstr} \sum_{\run=\start}^{\nRuns} \temp_{\run} \mbound_{\run}^{2},
\end{equation}
where we used the fact that $\tvnorm{\cdot} \leq \coef \norm{\cdot}$ so $\dnorm{\cdot} \leq \coef \supnorm{\cdot}$.
The bound \eqref{eq:reg-bound-stat} then follows by combining the above with \cref{lem:point2simple}.

Regarding the bound \eqref{eq:reg-bound-stat-powers}, we first note that this is not a pseudo-regret bound but a bona fide bound for the learner's \emph{expected regret} (so we cannot simply our point-dependent bound over $\point\in\points$).
In light of this, our first step will be to consider a ``uniform'' simple approximant for every $\point\in\points$.
To that end, building on an idea by \citet{BK99a} and \citet{KBTB15}, fix a shrinkage factor $\width>0$ and let $\points_{\width}(\point) = \setdef{\point + \width(\pointalt - \point)}{\pointalt\in\points} \subseteq \points$ denote the homothetic transformation that shrinks $\points$ to a fraction $\width$ of its original size and then transports it to $\point\in\points$.
By construction,
we have $\point \in \points_{\width}(\point) \subseteq \points$
and, moreover,
$\diam(\points_{\width}(\point)) = \width \diam(\points)$ and $\leb(\points_{\width}(\point)) = \width^{\vdim} \leb(\points)$.
Then, letting $\mu_{\point} \defeq \unif_{\points_{\width}(\point)}$ denote the uniform strategy supported on $\points_{\width}(\point)$, we get
\begin{equation}
\label{eq:reg-mean-point2simple}
\exof{\reg(\nRuns)}
	= \exof*{\max_{\point\in\points} \reg_{\point}(\nRuns)}
	\leq \exof*{\max_{\point\in\points} \reg_{\mu_{\point}}(\nRuns)}
		+ \width \lips \diam(\points) \nRuns,
\end{equation}
where, in the last step, we used \cref{lem:point2simple}.

Now, by \cref{prop:reg-simple}, we have
\begin{align}
\reg_{\mu_{\point}}(\nRuns)
	&\leq \frac{\hreg(\mu_{\point}) - \min\hreg}{\temp_{\nRuns+1}}
		+ \sum_{\run=\start}^{\nRuns} \braket{\error_{\run}}{\state_{\run} - \mu_{\point}}
		+ \frac{1}{2\hstr} \sum_{\run=\start}^{\nRuns} \temp_{\run} \dnorm{\model_{\run}}^{2}
	\notag\\
	&\leq \frac{\hvol(\width^{\vdim} \leb(\points)) - \hvol(\leb(\points))}{\temp_{\nRuns+1}}
		+ \sum_{\run=\start}^{\nRuns} \braket{\error_{\run}}{\state_{\run} - \mu_{\point}}
		+ \frac{\coef^{2}}{2\hstr} \sum_{\run=\start}^{\nRuns} \temp_{\run} \supnorm{\model_{\run}}^{2}.
\end{align}
and hence
\begin{equation}
\label{eq:reg-mean-inter}
\exof*{\max_{\point\in\points} \reg_{\mu_{\point}}(\nRuns)}
	\leq \frac{\hvol(\width^{\vdim} \leb(\points)) - \hvol(\leb(\points))}{\temp_{\nRuns+1}}
		+ \exof*{ \max_{\point\in\points} \sum_{\run=\start}^{\nRuns} \braket{\error_{\run}}{\state_{\run} - \mu_{\point}} }
		+ \frac{\coef^{2}}{2\hstr} \sum_{\run=\start}^{\nRuns} \temp_{\run} \mbound_{\run}^{2}.
\end{equation}
Thus, to proceed, it suffices to bound the second term of the above expression.

To do so, introduce the auxiliary process
\begin{equation}
\label{eq:DA-aux}
%\begin{aligned}
\daux_{\run+1}
	= \daux_{\run} - \noise_{\run},
%	\\
\quad
\aux_{\run+1}
	= \mirror(\temp_{\run+1} \daux_{\run+1}),
%\end{aligned}
\end{equation}
with $\aux_{\start} = \state_{\start}$.
We then have
\begin{align}
\label{eq:reg-aux}
\sum_{\run=\start}^{\nRuns} \braket{\error_{\run}}{\state_{\run} - \mu_{\point}}
	&= \sum_{\run=\start}^{\nRuns} \braket{\error_{\run}}{(\state_{\run} - \aux_{\run}) + (\aux_{\run} - \mu_{\point})}
%	+ \sum_{\run=\start}^{\nRuns} \braket{\error_{\run}}{\aux_{\run} - \mu_{\point}}
	\notag\\
	&= \sum_{\run=\start}^{\nRuns} \braket{\error_{\run}}{\state_{\run} - \aux_{\run}}
	+ \sum_{\run=\start}^{\nRuns} \braket{\bias_{\run}}{\aux_{\run} - \mu_{\point}}
	+ \sum_{\run=\start}^{\nRuns} \braket{\noise_{\run}}{\aux_{\run} - \mu_{\point}}
%	&= \sum_{\run=\start}^{\nRuns} \braket{\error_{\run}}{\aux_{\run} - \mu_{\point}}
%	+ \diam(\points) \sum_{\run=\start}^{\nRuns} \dnorm{\error_{\run}},
\end{align}
so it suffices to derive a bound for each of these terms.
This can be done as follows:
\begin{enumerate}
[leftmargin=2em]
\item
The first term of \eqref{eq:reg-aux} does not depend on $\point$, so we have
\begin{align}
\label{eq:reg-aux1}
\exof*{\max_{\point\in\points} \sum_{\run=\start}^{\nRuns} \braket{\error_{\run}}{\state_{\run} - \aux_{\run}}}
	&= \sum_{\run=\start}^{\nRuns}
		\exof*{\exof{\braket{\error_{\run}}{\state_{\run} - \aux_{\run}} \given \filter_{\run}}}
%	\notag\\
%	&= \sum_{\run=\start}^{\nRuns} \exof{\braket{\bias_{\run}}{\state_{\run} - \aux_{\run}}}
%%	\notag\\
	\leq 2 \bbound_{\run}
\end{align}
where, in the last step, we used the definition \eqref{eq:bias} of $\bbound_{\run}$ and the bound
\begin{equation}
\braket{\bias_{\run}}{\state_{\run} - \aux_{\run}}
	\leq \onenorm{\state_{\run} - \aux_{\run}} \supnorm{\bias_{\run}}
	\leq 2 \bbound_{\run}.
\end{equation}

\item
The second term of \eqref{eq:reg-aux} can be similarly bounded as
\begin{align}
\label{eq:reg-aux2}
\exof*{\max_{\point\in\points} \sum_{\run=\start}^{\nRuns} \braket{\bias_{\run}}{\aux_{\run} - \mu_{\point}}}
	\leq \exof*{\onenorm{\aux_{\run} - \mu_{\point}} \dnorm{\bias_{\run}}}
	\leq 2 \bbound_{\run}.
\end{align}

\item
The third term is more challenging;
the main idea will be to apply \cref{prop:reg-simple} on the sequnce $-\noise_{\run}$, $\run=\running$, viewed itself as a sequence of virtual payoff functions.
Doing just that, we get:
\begin{align}
\sum_{\run=\start}^{\nRuns} \braket{\noise_{\run}}{\aux_{\run} - \mu_{\point}}
	&\leq \frac{\hreg(\mu_{\point}) - \min\hreg}{\temp_{\nRuns+1}}
		+ \frac{1}{2\hstr} \sum_{\run=\start}^{\nRuns} \temp_{\run} \dnorm{\noise_{\run}}^{2}
	\notag\\
	&\leq \frac{\hvol(\width^{\vdim} \leb(\points)) - \hvol(\leb(\points))}{\temp_{\nRuns+1}}
		+ \frac{\coef^{2}}{2\hstr} \sum_{\run=\start}^{\nRuns} \temp_{\run} \supnorm{\noise_{\run}}^{2}.
\end{align}
Thus, after maximizing and taking expectations, we obtain
\begin{equation}
\label{eq:reg-aux3}
\exof*{\max_{\point\in\points} \braket{\noise_{\run}}{\aux_{\run} - \mu_{\point}}}
	\leq \frac{\hvol(\width^{\vdim} \leb(\points)) - \hvol(\leb(\points))}{\temp_{\nRuns+1}}
		+ \frac{\coef^{2}}{2\hstr} \sum_{\run=\start}^{\nRuns} \temp_{\run} \sdev_{\run}^{2}.
\end{equation}
\end{enumerate}

Therefore, plugging \cref{eq:reg-aux1,eq:reg-aux2,eq:reg-aux3} into \eqref{eq:reg-aux} and substituting the result to \eqref{eq:reg-mean-inter}, we finally get
\begin{equation}
\label{eq:reg-mean-inter2}
\exof*{\max_{\point\in\points} \reg_{\mu_{\point}}(\nRuns)}
	\leq 2\frac{\hvol(\width^{\vdim} \leb(\points)) - \hvol(\leb(\points))}{\temp_{\nRuns+1}}
		+ \frac{\coef^{2}}{2\hstr} \sum_{\run=\start}^{\nRuns} \temp_{\run} (\mbound_{\run}^{2} + \sdev_{\run}^{2}).
\end{equation}
The guarantee \eqref{eq:reg-bound-stat-powers} then follows by taking $\width^{\vdim} \leb(\points) = \nRuns^{-\vdim\dexp}$ for some $\dexp\geq0$
%(so $\diam(\points_{\point}(\width)) = \nRuns^{-\dexp}$)
and plugging everything back in \eqref{eq:reg-mean-point2simple}.
\end{proof}

%----------------------------------------------------------------------
%%% DYNAMIC
%----------------------------------------------------------------------
\subsection{Dynamic regret guarantees}

We now turn to the algorithm's dynamic regret guarantees, as encoded by \cref{thm:reg-dyn} (stated below for convenience):

\regdyn*

\begin{proof}[Proof of \cref{thm:reg-dyn}]
As we discussed in the main body of our paper, our proof strategy will be to decompose the horizon of play into $\nBatches$ virtual segments, estimate the learner's regret over each segment, and then compare the learner's regret \emph{per-segment} to the corresponding dynamic regret over said segment.
We stress here again that this partition is only made for the sake of the analysis, and does not involve restarting the algorithm \textendash\ \eg as in \citet{BGZ15}.

To make this precise, we first partition the interval $\runs = \window{1}{\nRuns}$ into $\nBatches$ contiguous segments $\runs_{\iBatch}$, $\iBatch=1,\dotsc,\nBatches$, each of length $\batch$ (except possibly the $\nBatches$-th one, which might be smaller).
More explicitly, take the window length to be of the form $\batch = \ceil{\nRuns^{\qexp}}$ for some constant $\qexp\in[0,1]$ to be determined later.
In this way, the number of windows is $\nBatches = \ceil{\nRuns/\batch} = \Theta(\nRuns^{1-\qexp})$ and the $\iBatch$-th window will be of the form $\runs_{\iBatch} = \window{(\iBatch-1)\batch+1}{\iBatch\batch}$ for all $\iBatch = 1,\dotsc,\nBatches-1$ (the value $\iBatch = \nBatches$ is excluded as the $\nBatches$-th window might be smaller).
For concision, we will denote the learner's static regret over the $\iBatch$-th window as $\reg(\runs_{\iBatch}) = \max_{\point\in\points} \sum_{\run\in\runs_{\iBatch}} \braket{\pay_{\run}}{\dirac_{\point} - \state_{\run}}$ (and likewise for its dynamic counterpart).

To proceed, let $\subruns \subseteq \runs$ be a sub-interval of $\runs$ and write $\sol_{\subruns} \in \argmax_{\point\in\points} \sum_{\runalt\in\subruns} \pay_{\runalt}(\point)$ for any action that is optimal on average over the interval $\subruns$.
To ease notation, we also write $\sol_{\run} \equiv \sol_{\{\run\}} \in \argmax_{\point\in\points} \pay_{\run}(\point)$ for any action that is optimal at time $\run$, and $\sol_{\iBatch} \equiv \sol_{\runs_{\iBatch}}$ for any action that is optimal on average over the $\iBatch$-th window.
Then, for all $\run\in\batch_{\iBatch}$, $\iBatch = \running,\nBatches$, we have
\begin{equation}
\braket{\pay_{\run}}{\dirac_{\sol_{\run}} - \state_{\run}}
	= \braket{\pay_{\run}}{\dirac_{\sol_{\iBatch}} - \state_{\run}}
%		+ \braket{\pay_{\run}}{\dirac_{\sol_{\run}} - \sol_{\iBatch}},
		+ \bracks{\pay_{\run}(\sol_{\run}) - \pay_{\run}(\sol_{\iBatch})}
\end{equation}
so the learner's dynamic regret over $\runs_{\iBatch}$ can be bounded as
\begin{align}
\label{eq:dyn=reg+var}
\dynreg(\runs_{\iBatch})
%	&= \sum_{\run\in\runs_{\iBatch}} \braket{\pay_{\run}}{\dirac_{\sol_{\run}} - \state_{\run}}
	= \sum_{\run\in\runs_{\iBatch}} \braket{\pay_{\run}}{\dirac_{\sol_{\iBatch}} - \state_{\run}}
%		+ \sum_{\run\in\runs_{\iBatch}} \braket{\pay_{\run}}{\dirac_{\sol_{\run}} - \dirac_{\sol_{\iBatch}}}
		+ \sum_{\run\in\runs_{\iBatch}} \bracks{\pay_{\run}(\sol_{\run}) - \pay_{\run}(\sol_{\iBatch})}
	= \reg(\runs_{\iBatch})
		+ \sum_{\run\in\runs_{\iBatch}} \bracks{\pay_{\run}(\sol_{\run}) - \pay_{\run}(\sol_{\iBatch})}.
\end{align}

Following a batch-comparison technique originally due to \citet{BGZ15}, let $\runstart_{\iBatch} = \min\runs_{\iBatch}$ denote the beginning of the $\iBatch$-th window, and let $\sol_{\runstart_{\iBatch}}$ denote a maximizer of the first payoff function encountered in the window $\runs_{\iBatch}$ (this choice could of course be arbitrary).
Thus, given that $\sol_{\iBatch}$ maximizes the per-window aggregate $\sum_{\run\in\runs_{\iBatch}} \pay_{\run}(\point)$, we obtain:
\begin{align}
\label{eq:varbound}
\sum_{\run\in\runs_{\iBatch}} \bracks{\pay_{\run}(\sol_{\run}) - \pay_{\run}(\sol_{\iBatch})}
	&\leq \sum_{\run\in\runs_{\iBatch}} \bracks{\pay_{\run}(\sol_{\run}) - \pay_{\run}(\sol_{\runstart_{\iBatch}})}
	\notag\\
	&\leq \abs{\runs_{\iBatch}} \max_{\run\in\runs_{\iBatch}} \bracks{\pay_{\run}(\sol_{\run}) - \pay_{\run}(\sol_{\runstart_{\iBatch}})}
%	\notag\\
	\leq 2\batch \tvar[\iBatch],
\end{align}
where we let $\tvar[\iBatch] = \sum_{\run\in\runs_{\iBatch}} \supnorm{\pay_{\run+1} - \pay_{\run}}$.
%(with the convention $\pay_{\max\runs_{\iBatch}+1} = \pay_{\max\runs_{\iBatch}}$ for the last summand.
In turn, combining \eqref{eq:varbound} with \eqref{eq:dyn=reg+var}, we get:
\begin{align}
\dynreg(\runs_{\iBatch})
	\leq \reg(\runs_{\iBatch})
		+ 2 \batch \tvar[\iBatch],
\end{align}
and hence, after summing over all windows:
\begin{align}
\label{eq:dyn=reg+var2}
\dynreg(\nRuns)
	\leq \sum_{\iBatch=\start}^{\nBatches} \reg(\runs_{\iBatch})
		+ 2 \batch \tvar.
\end{align}

Now \cref{thm:reg-stat} applied to the Hedge variant of \cref{alg:DA} readily yields
\begin{equation}
\exof{\reg(\runs_{\iBatch})}
	= \bigoh\parens*{
		(\iBatch\batch)^{\pexp}
		+ \batch^{1-\dexp}
		+ \sum_{\run\in\runs_{\iBatch}} \run^{-\bexp}
		+ \sum_{\run\in\runs_{\iBatch}} \run^{1+2\mexp-\pexp}}
\end{equation}
so, after summing over all windows, we have
\begin{align}
\label{eq:reg-sum}
\sum_{\iBatch=1}^{\nBatches} \exof{\reg(\runs_{\iBatch})}
	&= \bigoh\parens*{
		\batch^{\pexp} \sum_{\iBatch=1}^{\nBatches} \iBatch^{\pexp}
		+ \nBatches\batch^{1-\dexp}
		+ \sum_{\run=\start}^{\nRuns} \run^{-\bexp}
		+ \sum_{\run=\start}^{\nRuns} \run^{2\wexp-\pexp}
		} % end parens
	\notag\\
	&= \bigoh\parens*{
		\batch^{\pexp} \nBatches^{1+\pexp}
		+ \nBatches\batch^{1-\dexp}
		+ \nRuns^{1-\bexp}
		+ \nRuns^{1+2\wexp-\pexp}
		}. % end parens
\end{align}
%Since $\batch^{\pexp}\nBatches^{1+\pexp} = \bigoh(\nRuns^{\qexp\pexp} \nRuns^{(1-\qexp)(1+\pexp)}) = \bigoh(\nRuns^{1+\pexp-\qexp})$, the bound \eqref{eq:reg-mean-dyn} follows by setting the window length exponent equal to $\qexp = 2\pexp - 2\wexp$ in \eqref{eq:reg-gap-dyn2}.
Since $\batch = \bigoh(\nRuns^{\qexp})$ and $\nBatches = \bigoh(\nRuns/\batch) = \bigoh(\nRuns^{1-\qexp})$, we get
\begin{equation}
\batch^{\pexp} \nBatches^{1+\pexp}
	= \bigoh((\nBatches\batch)^{\pexp} \, \nBatches)
	= \bigoh(\nRuns^{\qexp\pexp} \nRuns^{(1-\qexp)(1+\pexp)})
	= \bigoh(\nRuns^{1+\pexp-\qexp})
\end{equation}
and, likewise
\begin{equation}
\nBatches\batch^{1-\dexp}
	= \bigoh(\nRuns \batch^{-\dexp})
	= \bigoh(\nRuns \nRuns^{-\qexp\dexp})
	= \bigoh(\nRuns^{1-\qexp\dexp}).
\end{equation}
Then, substituting in \eqref{eq:reg-sum} and \eqref{eq:dyn=reg+var2}, we finally get the dynamic regret bound
\begin{equation}
\exof{\dynreg(\nRuns)}
%\sum_{\iBatch=1}^{\nBatches} \exof{\reg(\runs_{\iBatch})}
	= \bigoh\parens*{
		\nRuns^{1+\pexp-\qexp}
		+ \nRuns^{1-\qexp\dexp}
		+ \nRuns^{1-\bexp}
		+ \nRuns^{1+2\wexp-\pexp}
		+ \nRuns^{\qexp}\tvar}. % end parens
\end{equation}
To balance the above expression, we take
$\qexp = 2\pexp - 2\wexp$ for the window size exponent (which calibrates the first and fourth terms in the sum above)
and
$\dexp = \bexp/\qexp = \bexp/(2\pexp-2\wexp)$ (for the second and the third).
In this way, we finally obtain
\begin{equation}
\exof{\dynreg(\nRuns)}
%\sum_{\iBatch=1}^{\nBatches} \exof{\reg(\runs_{\iBatch})}
	= \bigoh\parens*{
		\nRuns^{1-\bexp}
		+ \nRuns^{1+2\wexp-\pexp}
		+ \nRuns^{2\pexp-2\wexp}\tvar} % end parens
\end{equation}
and our proof is complete.
\end{proof}

%----------------------------------------------------------------------
%%% APP: BANDIT
%----------------------------------------------------------------------
\section{Derivations for the bandit framework}
\label{app:bandit}
%----------------------------------------------------------------------
%%% APP: BANDIT
%----------------------------------------------------------------------
% !TEX root = ./Main.tex

In this appendix, we aim at deriving guarantees for the Hedge variant of \cref{alg:BDA} using template bounds from \cref{app:regret}. We start by stating preliminary results that are used in the sequel.

\subsection{Preliminary results}
We first present a technical bound for the convex conjugate of the entropic regularizer (more on this below):

\begin{lemma}
%[LogSumExp bound]
\label{lem:logsumexp-bound}
For all $\dual, \dvec \in \dspace$, there exists $\xi \in [0,1]$ such that:
\begin{equation}
\log \left(\int_{\points} \exp(\dual + \dvec)\right) \leq
\log \left(\int_{\points} \exp(\dual)\right) +
\braket{\dvec}{\logit(\dual)} +
\frac{1}{2}\braket{\dvec^2}{\logit(\dual + \xi \dvec)}.
\end{equation}
\end{lemma}
\begin{proof}
Consider the function $\phi\from [0,1] \to \R$ with
%\begin{equation}
\(
\phi(t) = \log \left(\int_{\points} \exp(\dual + t\dvec)\right).
\)
%\end{equation}
By construction, $\phi(0) = \log \left(\int_{\points} \exp(\dual)\right)$ and $\phi(1) = \log \left(\int_{\points} \exp(\dual + \dvec)\right)$.
Thus, by a second-order Taylor expansion with Lagrange remainder, we have:
\begin{equation}
\label{eq:phi-taylor-cauchy}
\phi(1) = \phi(0) + \phi'(0) + \frac{1}{2}\phi''(\xi)
\end{equation}
for some $\xi\in[0,1]$.

Now, for all $t \in [0,1]$, $\phi'(t) = \frac{\int_{\points} \dvec \exp(\dual + t\dvec)}{\int_{\points} \exp(\dual + t\dvec)}$, which in turns gives 
\begin{equation}
\label{eq:first-derivative-phi}
	\phi'(0) = \frac{\int_{\points} \dvec \exp(\dual)}{\int_{\points} \exp(\dual)} = \braket{\dvec}{\logit(\dual)}.
\end{equation}
As for the second order derivative of $\phi$, we have for all $t \in [0,1]$:
\begin{align}
	\phi''(t) &= \frac{\partial}{\partial t} \left[\frac{\int_{\points} \dvec \exp(\dual + t\dvec)}{\int_{\points} \exp(\dual + t\dvec)}\right]
	\notag\\
	&= \frac{\int_\points \dvec^2 \exp(\dual + t\dvec)\int_\points \exp(\dual + t\dvec) - \left(\int_\points \dvec \exp(\dual + t\dvec)\right)^2}{\left(\int_\points \exp(\dual + t\dvec)\right)^2}
	\notag\\
	&\leq \frac{\int_\points \dvec^2 \exp(\dual + t\dvec)\int_\points \exp(\dual + t\dvec)}{\left(\int_\points \exp(\dual + t\dvec)\right)^2}
%	\notag\\
	= \frac{\int_\points \dvec^2 \exp(\dual + t\dvec)}{\int_\points \exp(\dual + t\dvec)}
\end{align}
Thus, for all $t \in [0,1]$, we get
\begin{equation}
\label{eq:second-derivative-phi}
		\phi''(t) \leq \braket{\dvec^2}{\logit(\dual + t \dvec)}.
\end{equation}
Our claim then follows by injecting \eqref{eq:first-derivative-phi} and \eqref{eq:second-derivative-phi} into \eqref{eq:phi-taylor-cauchy}.
\end{proof}

In the next lemma, we now present an expression of the Fenchel coupling in the specific case of the negentropy regularizer $\hreg(\primal) = \int_{\points}\primal \log\primal$.

\begin{lemma}
%[Fenchel coupling with negentropy regularizer]
\label{lem:fenchel-hedge}
In the case of the negentropy regularizer $\hreg(\primal) = \int_{\points}\primal \log\primal$, the Fenchel coupling for all $\dual \in \dspace$ and $\primal \in \dom\hreg$ is given by
\begin{equation}
\label{eq:fenchel-hedge}
\fench(\primal, \dual) = 
	\int_{\points}\primal \log\primal + 
	\log \left(\int_{\points} \exp(\dual)\right) - 
	\braket{\dual}{\primal}.
\end{equation}
\end{lemma}
\begin{proof}
We remind the general expression of the Fenchel coupling given in \eqref{eq:Fench}:
\begin{equation}
\fench(\primal,\dual)=\hreg(\primal)+\hconj(\dual)-\braket{\dual}{\primal}
	\quad
	\text{for all $\primal \in \dom\hreg,\dual \in \dspace$},
\end{equation}
where $\hconj(\dual) = \max_{\primal\in\vecspace} \{ \braket{\dual}{\primal} - \hreg(\primal) \}$.
In the case of the negentropy regularizer $\hreg(\primal) = \int_\points \primal \log \primal$, we have that
%\begin{equation}
\(
%\logit(\dual) =
\argmax_{\primal\in\vecspace} \{ \braket{\dual}{\primal} - \hreg(\primal) \}
	= \logit(\dual)
\)
%\end{equation}
and
\begin{equation}
\label{eq:hconj-for-proof}
	\hconj(\dual) = \braket{\dual}{\logit(\dual)} - \hreg(\logit(\dual)).
\end{equation}
%Moreover,
%\begin{align*}
%	\logit(\dual) = \frac{\exp(\dual)}{\int_\points \exp(\dual)}, \ \text{ and } \ 
%	\log \logit(\dual) = \dual - \log\left(\int_\points \exp(\dual)\right).
%\end{align*}
Combining the above, we then get:
% \eqref{eq:hconj-for-proof}, we write:
\begin{align*}
	\hconj(\dual) &= \int_\points \dual \logit(\dual) - \int_\points \logit(\dual) \log \logit(\dual) \\
	&= \int_\points \dual \logit(\dual) - \int_\points \logit(\dual)\dual + \int_\points \log\left(\int_\points \exp(\dual)\right)\logit(\dual)\\
	&= \log\left(\int_\points \exp(\dual)\right).
\end{align*}
which, combined with \eqref{eq:Fench}, delivers \eqref{eq:fenchel-hedge}.
\end{proof}

Finally we state a result enabling to control the difference between the regret $\reg(\nRuns)$ and $\regalt(\nRuns)$ induced respectively by two policies $\state_\run$ and $\tilde{\state}_\run$ against the same rewards and models. 

\begin{lemma}
\label{lem:regret-alternative-policy}
For $\run=1, \dots, \nRuns$, let $\state_\run$, $\tilde{\state}_\run$ be two policies with respective regret $\reg(\nRuns)$ and $\regalt(\nRuns)$ against a given sequence of models $(\model_\run)_\run$ for the rewards $(\pay_\run)_\run$.
Then:
\begin{equation}
\reg(\nRuns) \leq \regalt(\nRuns) + \sum\limits_{\run = 1}^\nRuns \supnorm{\state_\run - \tilde{\state}_\run}.
\end{equation}
\end{lemma}
\begin{proof}
See \citet[Chap.~6]{AS19}.
\end{proof}

\subsection{Hedge-specific bounds}
We are now ready to adapt the template bound of \cref{lem:template} to the Hedge case. 
\begin{lemma}
\label{lem:template-bandits}
Assuming the regularizer $\hreg$ is the negentropy $\hreg(\primal) = \int_{\points}\primal \log\primal$, and that the mirror map $\mirror$ corresponds to the logit operator $\logit$, there exists $\xi \in [0,1]$ such that, for all $\simple \in \simples$ the policy \eqref{eq:DA} enjoys the bound:
\begin{equation}
	\label{eq:lem-template-bandits}
	\energy_{\run+1}
	\leq \energy_{\run}
		+ \braket{\model_{\run}}{\state_{\run} - \simple}
		+ \parens*{\frac{1}{\temp_{\run+1}} - \frac{1}{\temp_{\run}}} \bracks{\hreg(\simple) - \min\hreg}
		+ \frac{\temp_{\run}}{2} G_\run(\xi)^2.
\end{equation}
where for all $\xi \in [0,1]$, $G_\run(\xi)^2 = \braket{\logit(\eta_\run y_\run + \xi \eta_\run \model_\run)}{\model_\run^2}$.
\end{lemma}
\begin{proof}
We know from \cref{lem:template-bandits} that the policy \eqref{eq:DA} enjoys the bound:
\begin{equation}
\label{eq:general-template-bound}
	\energy_{\run+1}
	\leq \energy_{\run}
		+ \braket{\model_{\run}}{\state_{\run} - \simple}
		+ \parens*{\frac{1}{\temp_{\run+1}} - \frac{1}{\temp_{\run}}} \bracks{\hreg(\simple) - \min\hreg}
		+ \frac{1}{\temp_{\run}} \fench(\state_{\run},\temp_{\run}\score_{\run+1}).
\end{equation}
The following lemma will help us handle the Fenchel coupling term in \eqref{eq:general-template-bound}
\begin{lemma}
\label{lem:fenchel-bound-hedge}
For a given $\run$ in the policy \eqref{eq:DA}, there exists $\xi \in [0,1]$ such that the following bounds holds:
\begin{equation}
	\fench(\state_\run, \eta_\run y_{\run+1}) \leq \frac{\eta_\run^2}{2}G_\run(\xi)^2.
\end{equation}
\end{lemma}
Injecting the result given in \cref{lem:fenchel-bound-hedge} in \cref{eq:general-template-bound} yields the stated claim.
\end{proof}

Moving forward, we are only left to prove \cref{lem:fenchel-bound-hedge}.
\begin{proof}
Since we are in the case of the negentropy regularizer, \cref{lem:fenchel-hedge} enables to rewrite the Fenchel coupling term of \eqref{eq:general-template-bound} as:
\begin{equation}
	\label{eq:fenchel-hedge-for-proof}
	\fench(\state_{\run},\temp_{\run}\score_{\run+1}) = 
	\int_{\points}\state_{\run} \log\state_{\run} + 
	\log \left(\int_{\points} \exp(\temp_{\run}\score_{\run+1})\right) - 
	\braket{\temp_{\run}\score_{\run+1}}{\state_{\run}}.
\end{equation}
Injecting $\score_{\run + 1} = \score_{\run} + \model_\run$ in \eqref{eq:fenchel-hedge-for-proof} yields:
\begin{align}
	\fench(\state_{\run},\temp_{\run}\score_{\run+1}) &= \int_{\points}\state_{\run} \log\state_{\run} + 
	\log \left(\int_{\points} \exp(\temp_{\run}\score_{\run} + \temp_\run\model_\run)\right) - \braket{\temp_{\run}\score_{\run}}{\state_{\run}} 
	- \braket{\temp_{\run}\model_{\run}}{\state_{\run}}
	\notag\\
	&= \fench(\state_{\run},\temp_{\run}\score_{\run}) + \left(\log \int_{\points} \exp(\temp_{\run}\score_{\run} + \temp_\run\model_\run) - \log \int_{\points} \exp(\temp_{\run}\score_{\run}) \right) - \braket{\temp_{\run}\model_{\run}}{\state_{\run}}
	\notag\\
	&= \log \left(\int_{\points} \exp(\temp_{\run}\score_{\run} + \temp_\run\model_\run)\right) - \log \left(\int_{\points} \exp(\temp_{\run}\score_{\run})\right) - \braket{\temp_{\run}\model_{\run}}{\state_{\run}}
\end{align}
where we used the fact that $\fench(\state_\run, \eta_\run y_\run) = 0$.

Now, by \cref{lem:logsumexp-bound} applied to $\dual \gets \eta_\run \score_\run$ and $\dvec \gets \eta_\run \model_\run$, there exists $\xi \in [0,1]$ such that
\begin{equation}
\label{eq:logsumexp-for-proof}
	\log \left(\int_{\points} \exp(\eta_\run \score_\run + \eta_\run \model_\run)\right) \leq
\log \left(\int_{\points} \exp(\eta_\run \score_\run)\right) +
\eta_\run \braket{\model_\run}{\logit(\eta_\run \score_\run)} +
\frac{\eta_\run^2}{2}\braket{\model_\run^2}{\logit(\eta_\run \score_\run + \xi \eta_\run \model_\run)},
\end{equation}
where we used the fact that $\state_\run = \logit(\eta_\run \score_\run)$.
Our claim then follows by injecting \eqref{eq:logsumexp-for-proof} into our prior expression for the Fenchel coupling $\fench(\state_{\run},\temp_{\run}\score_{\run+1})$ in the case of the Hedge variant.
\end{proof}

\begin{proposition}
\label{prop:regstat-bandits-hedge}
If we run the Hedge variant of \cref{alg:DA}, there exists a sequence $\xi_\run \in [0,1]$ such that:
\begin{equation}
\label{eq:reg-bound-stat-bandits}
\exof{\reg_{\point}(\nRuns)}
	\leq \frac{\log(\leb(\points) / \leb(\cvx))}{\temp_{\nRuns+1}}
	+ \lips\diam(\cvx) \nRuns
	+ 2 \insum_{\run=\start}^{\nRuns} \bbound_{\run}
	+ \frac{1}{2} \insum_{\run=\start}^{\nRuns} \temp_{\run} \exof{G_{\run}(\xi_t)^{2} \given \filter_{\run}},
\end{equation}
where $\cvx$ is a convex neighborhood of $\point$ in $\points$.
\end{proposition}
\begin{proof}
This result is obtained by using the template bound given in \cref{lem:template-bandits}, then by proceeding exactly as in the proofs of \cref{prop:reg-simple} and \cref{thm:reg-stat}.
\end{proof}

We stress here that \cref{prop:regstat-bandits-hedge} \emph{does not correspond} to the Hedge instantiation \cref{thm:reg-stat}.
Indeed, the second order term $\frac{1}{2} \insum_{\run=\start}^{\nRuns} \temp_{\run} \exof{G_{\run}(\xi_t)^{2} \given \filter_{\run}}$ builds on results that are specific to Hedge, and is a priori considerably sharper than $\frac{\coef^{2}}{2\hstr} \insum_{\run=\start}^{\nRuns} \temp_{\run} \mbound_{\run}^{2}$, the second order term of \cref{thm:reg-stat}.

\subsection{Guarantees for \cref{alg:BDA}}
For clarity, we begin by reminding the specific assumptions relative to \cref{alg:BDA}.
In particular, we are still considering throughout a dual averaging policy \eqref{eq:DA} with a negentropy regularizer. We additionally assume that at each round $\run$, we receive a model $\model_\run$ built according to the ``smoothing'' approach described in \cref{sec:bandit} where for all $\run$:
\begin{equation}
\model_{\run}(\point)
	= \kerfun_{\run}(\choice_{\run},\point)
		\cdot \pay_{\run}(\choice_{\run}) / \state_{\run}(\choice_{\run})
\end{equation}
where
$\kerfun_{\run} \from \points\times\points \to \R$ is a (time-varying) \emph{smoothing kernel}, \ie $\int_{\points} \kerfun_{\run}(\point,\pointalt) \dd\pointalt = 1$ for all $\point\in\points$.
%[As before, $\choice_{\run}$ is the action chosen by the learner at stage $\run$ based on the probability distribution $\state_{\run}$.]
For concreteness (and sampling efficiency), we will assume that payoffs now take values in $[0,1]$, and we will focus on simple kernels that are supported on a neighborhood $\nhd_{\width}(\point) = \ball_{\width}(\point) \cap \points$ of $\point$ in $\points$ and are constant therein, \ie $\kerfun^{\width}(\point,\pointalt) = [\leb(\nhd_{\width}(\point))]^{-1} \oneof{\norm{\pointalt - \point} \leq \width}$.\\
we incorporate in $\state_{\run}$ an explicit exploration term of the form $\mix_{\run}/\leb(\points)$.\\

Under these assumptions, we may now bound both the bias and variance terms in \eqref{eq:reg-bound-stat-bandits}.

\begin{lemma}
\label{lem:bias-and-variance-bounds-bandits}
The following inequality holds, where $\lips$ is a uniform Lipschitz coefficient for the reward functions $\pay_\run$ \textpar{as described in \cref{asm:loss}}
\begin{equation}
	\label{eq:bias-bound-bandits}
	\bbound_\run \leq \lips \delta_\run.
\end{equation}
Moreover, there exists a constant $C_\points$ (depending only on the set $\points$) such that:
\begin{align}
	\label{eq:variance-bound-bandits}
	\underset{\xi \in [0,1]}{\sup}\exof{G_{\run}(\xi)^{2} \given \filter_{\run}} &\leq C_\points \delta_\run^{-\vdim} \epsilon_\run^{-1}.
\end{align}
\end{lemma}
Note that bounding the second order term of \cref{thm:reg-stat} under the same assumptions would have yielded a $\delta_\run^{-2\vdim}$ factor instead of $\delta_\run^{-\vdim}$, which is a strictly weaker result!
\begin{proof}
We first prove \eqref{eq:bias-bound-bandits}.
%It appears from \eqref{eq:bias} that a uniform bound for $\exof{\model_\run - \pay_\run \given \filter_{\run}}$ is a valid definition of $\bbound_\run$.
%Let $\point \in \points$.
Using the fact that
\(
\pay_\run(\point) = \int_\points \pay_\run(\point)\kerfun_{\run}(\choice_{\run},\point)d\choice_\run,
\)
we obtain:
\begin{align}
\abs{\exof{\model_\run(\point) - \pay_\run(\point) \given \filter_{\run}}}
	&= \abs*{
		\int_\points \kerfun_{\run}(\choice_{\run},\point)
		\frac{\pay_{\run}(\choice_{\run})}{\state_{\run}(\choice_{\run})} \state_{\run}(\choice_{\run}) d\choice_\run - \int_\points \pay_\run(\point)\kerfun_{\run}(\choice_{\run},\point) d\choice_\run
		}
		\notag\\
		&= \abs*{
		\int_\points(\pay_{\run}(\choice_{\run}) - \pay_\run(\point) )\kerfun_{\run}(\choice_{\run},\point)d\choice_\run
		}
		\notag\\
		&= [\leb(\nhd_{\width_\run}(\point))]^{-1}
			\abs*{
			\int_\points \oneof{\norm{\pointalt - \point} \leq \width_\run}(\pay_{\run}(\choice_{\run}) - \pay_\run(\point) ) d\choice_\run
			}
		\notag\\
		&\leq \lips [\leb(\nhd_{\width_\run}(\point))]^{-1}
			\underbrace{\int_\points \oneof{\norm{\pointalt - \point} \leq \width_\run}(\pay_{\run}\norm{\choice_{\run}) - \pay_\run(\point)}}_{\leq \leb(\nhd_{\width_\run}(\point)) \width_\run}
%		\notag\\
		\leq \lips \width_\run.
\end{align}
This bound is uniform (does not depend on the point $\point$), and thus implies the stated inequality for $\bbound_\run$.

We now turn to \eqref{eq:variance-bound-bandits}.
To that end, let $\xi \in [0,1]$. We will prove a uniform bound on $\exof{G_{\run}(\xi)^{2} \given \filter_{\run}}$. As a preliminary it is capital to note that, $\points$ being convex compact, there exists constants $C_\points^M$ and $C_\points^m$ such that for all $\point \in \points$,
$$
C_\points^m \delta_\run^\vdim \leq [\leb(\nhd_{\width_\run}(\point))] \leq C^M_\points \delta_\run^\vdim.
$$
Now, using $G_\run(\xi)^2 = \braket{\logit(\eta_\run y_\run + \xi \eta_\run \model_\run)}{\model_\run^2}$ and $\model_{\run}(\point) = \kerfun_{\run}(\choice_{\run},\point) \cdot \pay_{\run}(\choice_{\run}) / \state_{\run}(\choice_{\run})$, we may write:
\begin{align}
\exof*{G_{\run}(\xi)^{2} \given \filter_{\run}}
	&= \exof*{\int_\points \logit(\eta_\run y_\run + \xi \eta_\run \model_\run)(\point) \kerfun_{\run}(\choice_{\run},\point)^2
		 \frac{\pay_{\run}(\choice_{\run})^2}{\state_{\run}(\choice_{\run})^2} d\point \given \filter_{\run}}
	\notag\\
	&\leq \int_\points 	\frac{\overbrace{\pay_{\run}(\pointalt)^2}^{\leq 1}}{\state_{\run}(\pointalt)^2} \state_\run(\pointalt) \left(\int_\points \logit(\eta_\run y_\run + \xi \eta_\run \model_\run)(\point) \kerfun_{\run}(\pointalt,\point)^2 d\point
	\right)d\pointalt
	\notag\\
	&\leq \int_\points \frac{1}{\underbrace{\state_{\run}(\pointalt)}_{\geq \mix_{\run}/\leb(\points)}} \underbrace{[\leb(\nhd_{\width_\run}(\pointalt))]^{-2}}_{\leq (C_\points^m)^2 \width^{-2\vdim}}\left(\int_\points \logit(\eta_\run y_\run + \xi \eta_\run \model_\run)(\point) \oneof{\norm{\pointalt - \point} \leq \width_\run} d\point
	\right)d\pointalt
	\notag\\
	&\leq \frac{\leb(\points)}{(C_\points^m)^2 \width_\run^{2\vdim}\mix_\run} \int_\points \left(\int_\points \logit(\eta_\run y_\run + \xi \eta_\run \model_\run)(\point) \oneof{\norm{\pointalt - \point} \leq \width_\run} d\point
	\right)d\pointalt
	\notag\\
	&=\frac{\leb(\points)}{(C_\points^m)^2 \width_\run^{2\vdim}\mix_\run} \int_\points \left(\int_\points \logit(\eta_\run y_\run + \xi \eta_\run \model_\run)(\point) \oneof{\norm{\pointalt - \point} \leq \width_\run} d\pointalt
	\right)d\point \ \ \text{(Fubini)}
	\notag\\
	&= \frac{\leb(\points)}{(C_\points^m)^2 \width_\run^{2\vdim}\mix_\run} \int_\points \logit(\eta_\run y_\run + \xi \eta_\run \model_\run)(\point) \underbrace{\left(\int_\points\oneof{\norm{\pointalt - \point} \leq \width_\run} d\pointalt
	\right)}_{= \leb(\nhd_{\width_\run}(\point)) \leq C_\points^M \width_t^\vdim \leb(\points)} d\point
	\notag\\
	&\leq \frac{\leb(\points)}{(C_\points^m)^2 \width_\run^{2\vdim}\mix_\run} C_\points^M \width_t^\vdim \underbrace{\left(\int_\points \logit(\eta_\run y_\run + \xi \eta_\run \model_\run)(\point) d\point \right)}_{=1}
	\notag\\
	&= \left(\frac{\leb(\points) C_\points^M}{(C_\points^m)^2}\right) \width_\run^{-\vdim}\mix_\run^{-1}.
\end{align}
This bound depends only on $\points$, and is notably independent on $\xi \in [0,1]$. The result \eqref{eq:variance-bound-bandits} follows directly.
\end{proof}

We are now ready to prove \cref{prop:reg-stat-BDA} and \cref{eq:reg-dyn-BDA}.\\

\regstatbandits*

\begin{proof}
Let us consider a slight modification of \cref{alg:BDA} in which
\begin{itemize}
	\item
	The models $(\model_\run)$ received by the learner are the same models than those generated by running \cref{alg:BDA},
	\item
	At each round $\run$, the action $\choice_\run$ is sampled according to $\tilde{\state}_\run = \logit(\eta_\run y_\run)$ (without taking into account the explicit exploration term).
\end{itemize}
The regret of this algorithm may be bounded using the Hedge template bound stated in \cref{prop:regstat-bandits-hedge}, since we are indeed considering the regret induced by Hedge against the sequence of reward models $\model_\run$\footnote{Even though these models were generated by \cref{alg:BDA}, which does not exactly corresponds to Hedge}.
Then, writing $\widetilde{\reg}(\nRuns)$ for the regret induced by the policy $(\tilde{\state}_\run)_\run$, we get
\begin{equation}
	\label{eq:alg2-hedge-bound-1}
	\exof{\widetilde{\reg}_{\point}(\nRuns)}
	\leq \frac{\log(\leb(\points) / \leb(\cvx))}{\temp_{\nRuns+1}}
	+ \lips\diam(\cvx) \nRuns
	+ 2 \insum_{\run=\start}^{\nRuns} \bbound_{\run}
	+ \frac{1}{2} \insum_{\run=\start}^{\nRuns} \temp_{\run} \exof{G_{\run}(\xi_t)^{2} \given \filter_{\run}}.
\end{equation}
Using the bounds presented in \cref{lem:bias-and-variance-bounds-bandits} we then get:
\begin{equation}
	\label{eq:alg2-hedge-bound-2}
	\exof{\widetilde{\reg}_{\point}(\nRuns)}
	\leq \frac{\log(\leb(\points) / \leb(\cvx))}{\temp_{\nRuns+1}}
	+ \lips\diam(\cvx) \nRuns
	+ 2 \lips \insum_{\run=\start}^{\nRuns} \width_\run
	+ \frac{1}{2} C_\points \insum_{\run=\start}^{\nRuns} \temp_{\run} \delta_\run^{-\vdim} \epsilon_\run^{-1}.
\end{equation}

We are however interested in guarantees for \cref{alg:BDA}, in which we play with the policy $(\state_\run)_\run$, which slightly differs from the Hedge policy $(\tilde{\state}_\run)_\run$.
To that end, \cref{lem:regret-alternative-policy} enables us to bound the difference between the regrets $\reg{\nRuns}$ and $\widetilde{\reg}(\nRuns)$, induced by $(\state_\run)_\run$ and $(\tilde{\state}_\run)_\run$ respectively. Namely we can write:
\begin{equation}
	\label{eq:alg2-and-hedge-regret-control-1}
	\reg{\nRuns} \leq \widetilde{\reg}(\nRuns) + \sum_{\run = 1}^T \supnorm{\tilde{\state}_\run - \state_\run}.
\end{equation}

For any $\run \geq 1$, $\point \in \points$, we have
\begin{align}
\abs*{\tilde{\state}_\run(\point) - \state_\run(\point)}
	= \abs*{\tilde{\state}_\run(\point) - (1 - \mix_\run)\tilde{\state}_\run(\point) - \frac{\mix_\run}{\leb(\points)}}
%	\notag\\
	= \mix_\run \cdot \abs*{\tilde{\state}_\run - 1 / \leb(\points)}
%	\notag\\
	\leq \mix_\run \left(1 + \frac{1}{\leb(\points)}\right).
\end{align}
Injecting this in \eqref{eq:alg2-and-hedge-regret-control-1} we get
\begin{equation}
	\label{eq:alg2-and-hedge-regret-control-2}
	\reg{\nRuns} \leq \widetilde{\reg}(\nRuns) + \left(1 + \frac{1}{\leb(\points)}\right)\sum_{\run = 1}^T \epsilon_\run.
\end{equation}
Finally, combining \eqref{eq:alg2-and-hedge-regret-control-2} with \eqref{eq:alg2-hedge-bound-2} yields:
\begin{align}
	\label{eq:alg2-regret-bound-1}
\exof{\widetilde{\reg}_{\point}(\nRuns)}
	&\leq \frac{\log(\leb(\points) / \leb(\cvx))}{\temp_{\nRuns+1}}
	+ \lips\diam(\cvx) \nRuns
	\notag\\
	&+ 2 \lips \insum_{\run=\start}^{\nRuns} \width_\run
	+ \frac{C_\points}{2} \insum_{\run=\start}^{\nRuns} \temp_{\run} \delta_\run^{-\vdim} \epsilon_\run^{-1} + \left(1 + \frac{1}{\leb(\points)}\right)\insum_{\run = 1}^\nRuns \epsilon_\run.
\end{align}

Now, using the same reasoning as in the proof of \cref{thm:reg-stat} with regards to the set $\cvx$, and using $\temp_{\run} \propto 1/\run^{\pexp}$, $\width_{\run} \propto 1/\run^{\wexp}$ and $\mix_{\run} \propto 1/\run^{\bexp}$ straightforwardly gives:
$$
\exof{\reg(\nRuns)}
= \bigoh(\nRuns^{\pexp} + \nRuns^{1-\wexp} + \nRuns^{1-\bexp} + \nRuns^{1+\vdim\wexp+\bexp-\pexp}).
$$
Finally, $\pexp=(\vdim+2)/(\vdim+3)$ and $\mexp = \bexp = 1/(\vdim+3)$ gives the optimal bound:
$$
\exof{\reg(\nRuns)} = \bigoh(\nRuns^{\frac{\vdim+2}{\vdim+3}}).
$$
\end{proof}

\regdynbandits*

\begin{proof}
We use the same virtual segmentation as in the proof of \cref{thm:reg-dyn}.
As a reminder, this means that we partition the interval $\runs = \window{1}{\nRuns}$ into $\nBatches$ contiguous segments $\runs_{\iBatch}$, $\iBatch=1,\dotsc,\nBatches$, each of length $\batch$ (except possibly the $\nBatches$-th one, which might be smaller).
More explicitly, take the window length to be of the form $\batch = \ceil{\nRuns^{\qexp}}$ for some constant $\qexp\in[0,1]$ to be determined later.
In this way, the number of windows is $\nBatches = \ceil{\nRuns/\batch} = \Theta(\nRuns^{1-\qexp})$ and the $\iBatch$-th window will be of the form $\runs_{\iBatch} = \window{(\iBatch-1)\batch+1}{\iBatch\batch}$ for all $\iBatch = 1,\dotsc,\nBatches-1$ (the value $\iBatch = \nBatches$ is excluded as the $\nBatches$-th window might be smaller).
For concision, we will denote the learner's static regret over the $\iBatch$-th window as $\reg(\runs_{\iBatch}) = \max_{\point\in\points} \sum_{\run\in\runs_{\iBatch}} \braket{\pay_{\run}}{\dirac_{\point} - \state_{\run}}$ (and likewise for its dynamic counterpart).

Following the proof of \cref{thm:reg-dyn} up to \eqref{eq:dyn=reg+var2}, we can still write in our bandit setting:
\begin{align}
\label{eq:dyn=reg+var2-bandits}
\dynreg(\nRuns)
	\leq \sum_{\iBatch=\start}^{\nBatches} \reg(\runs_{\iBatch})
		+ 2 \batch \tvar.
\end{align}

Now \cref{prop:reg-stat-BDA} applied to the Hedge variant of \cref{alg:BDA} readily yields
\begin{equation}
\exof{\reg(\runs_{\iBatch})}
	= \bigoh\parens*{
		(\iBatch\batch)^{\pexp}
		+ \sum_{\run\in\runs_{\iBatch}} \run^{-\bexp}
		+ \sum_{\run\in\runs_{\iBatch}} \run^{-\mexp}
		+ \sum_{\run\in\runs_{\iBatch}} \run^{\bexp+\vdim\mexp-\pexp}}
\end{equation}
so, after summing over all windows, we have
\begin{align}
\label{eq:reg-sum-bandits}
\sum_{\iBatch=1}^{\nBatches} \exof{\reg(\runs_{\iBatch})}
	&= \bigoh\parens*{
		\batch^{\pexp} \sum_{\iBatch=1}^{\nBatches} \iBatch^{\pexp}
		+ \sum_{\run=\start}^{\nRuns} \run^{-\bexp}
		+ \sum_{\run=\start}^{\nRuns} \run^{-\mexp}
		+ \sum_{\run=\start}^{\nRuns} \run^{\bexp +\vdim \mexp-\pexp}
		} % end parens
	\notag\\
	&= \bigoh\parens*{
		\batch^{\pexp} \nBatches^{1+\pexp}
		+ \nRuns^{1-\bexp}
		+ \nRuns^{1-\mexp}
		+ \nRuns^{1+\bexp+\vdim\mexp-\pexp}
		}. % end parens
\end{align}
%Since $\batch^{\pexp}\nBatches^{1+\pexp} = \bigoh(\nRuns^{\qexp\pexp} \nRuns^{(1-\qexp)(1+\pexp)}) = \bigoh(\nRuns^{1+\pexp-\qexp})$, the bound \eqref{eq:reg-mean-dyn} follows by setting the window length exponent equal to $\qexp = 2\pexp - 2\wexp$ in \eqref{eq:reg-gap-dyn2}.
Since $\batch = \bigoh(\nRuns^{\qexp})$ and $\nBatches = \bigoh(\nRuns/\batch) = \bigoh(\nRuns^{1-\qexp})$, we get
\begin{equation}
\batch^{\pexp} \nBatches^{1+\pexp}
	= \bigoh((\nBatches\batch)^{\pexp} \, \nBatches)
	= \bigoh(\nRuns^{\qexp\pexp} \nRuns^{(1-\qexp)(1+\pexp)})
	= \bigoh(\nRuns^{1+\pexp-\qexp}).
\end{equation}
Then, substituting in \eqref{eq:reg-sum-bandits} and \eqref{eq:dyn=reg+var2-bandits}, we finally get the dynamic regret bound
\begin{equation}
\exof{\dynreg(\nRuns)}
%\sum_{\iBatch=1}^{\nBatches} \exof{\reg(\runs_{\iBatch})}
	= \bigoh\parens*{
		\nRuns^{1+\pexp-\qexp}
		+ \nRuns^{1-\bexp}
		+ \nRuns^{1-\mexp}
		+ \nRuns^{1+\bexp+\vdim\mexp-\pexp}
		+ \nRuns^{\qexp}\tvar}. % end parens
\end{equation}
To balance the above expression, we take
$\qexp = 2\pexp - \bexp - \vdim \mexp$ for the window size exponent (which calibrates the first and fourth terms in the sum above).
In this way, we finally obtain
\begin{equation}
\exof{\dynreg(\nRuns)}
%\sum_{\iBatch=1}^{\nBatches} \exof{\reg(\runs_{\iBatch})}
	= \bigoh\parens*{
		\nRuns^{1+\vdim\wexp+\bexp-\pexp}
		+ \nRuns^{1-\bexp}
		+ \nRuns^{1-\wexp}
		+ \nRuns^{2\pexp-\vdim\wexp-\bexp}\tvar} % end parens
\end{equation}
and our proof is complete.
\end{proof}

%----------------------------------------------------------------------
%%% APP: NUMERICS
%----------------------------------------------------------------------
\section{Numerical experiments}
\label{app:numerics}
%----------------------------------------------------------------------
%%% APP: NUMERICS
%----------------------------------------------------------------------
% !TEX root = ./Main.tex

%----------------------------------------------------------------------
%%% BASIC
%----------------------------------------------------------------------
%\subsection{Numerics}

Our aim in this appendix is to provide some numerical illustrations of the theory presented in the rest of our paper.
All numerical experiments were run on a machine with 48 CPUs (Intel(R) Xeon(R) Gold 6146 CPU @ 3.20GHz), with 2 Threads per core, and 500Go of RAM.
For a simulation horizon of $\nRuns=2\times 10^{5}$, we choose a reward function $[0,1]$ that is a linear combination of trigonometric terms with different frequencies and amplitudes, arbitrarily drawn.
Because of this analytic expression, we are able to calculte the learner's best action in hindsight (or instantaneously) and plot the relevant regret curves.
%However, we stress that this setting is more a stationary bandit than a proper adversarial one.
%Nevertheless, once fixed the reward function 
%$$
%[0,1] \ni x \mapsto \frac{1}{2} + \frac{4\sin(4  (2 x - 1)) + 3 \cos(10 (2 x - 1))}{14}.
%$$

For illustration purposes, we compared $2$ strategies, called ``Grid'' and ``Kernel''.
The ``Kernel'' method is as outlined in \cref{sec:bandit} (\cf \cref{alg:BDA}) with parameters described below.
The ``Grid'' method involves partitioning the search space into a grid of a given mesh-size (a hyperparameter of the algorithm), and then treating the problem as a finite-armed bandit;
in particular, the ``Grid'' strategy employs the EXP3 algorithm \citep{ACBF02} with rewards sampled at the grid points.

In \cref{fig:regret}, we plot the mean regret for both algorithms, with different hyperparameters, over $\nRuns$ iterations.
The confidence intervals are represented by the shaded areas, which corresponds to the mean value of the regret modulated by the standard deviation of our sample runs of each algorithm (computed on 92 initialization seeds for sampling, kept constant across different runs for control validation).

\begin{figure}[tbp]
    \centering
    \includegraphics[width=0.7\textwidth]{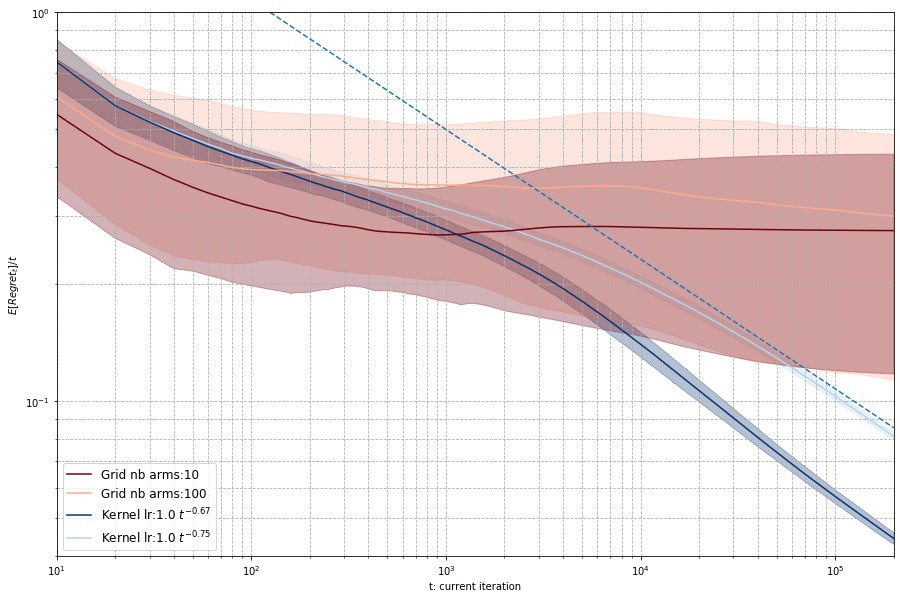}
    \caption{\textbf{Expected average regret}, averaged on 92 realizations for each algorithm (solid line). The variance is presented (shaded area) where we add and remove the standard deviation (computed on the 92 seeds) from the mean. Finally, the theoretical regret bound is displayed (dashed line).}
\label{fig:regret}
\end{figure}

\begin{figure}[tbp]
    \centering
    \includegraphics[width=0.9\textwidth]{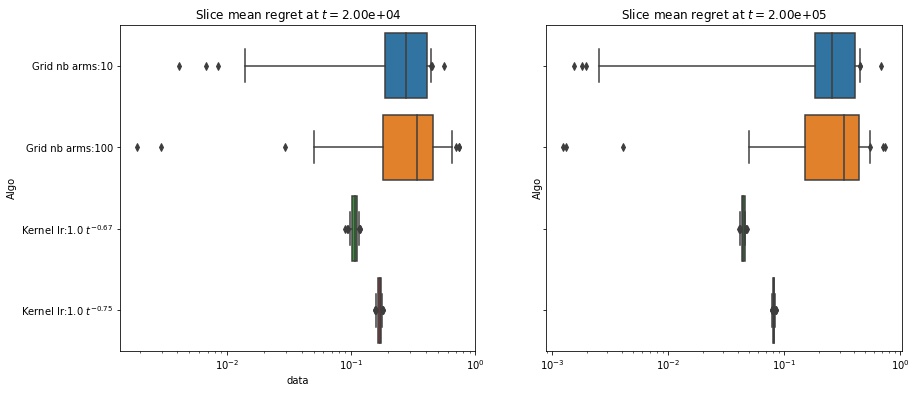}
    \caption{\textbf{Two slices of the mean regret}, averaged on 92 realizations for each algorithm (solid line). Whisker at 5-95\% CI , boxes at 25-75\% CI and median displayed with vertical bars.}
\label{fig:slices}
\end{figure}

The dashed line represent in the figure corresponds to the theoretical regret bound of $\nRuns^{-\frac{1}{3}}$, which is the expected regret bound of the Kernel algorithm mean regret (without explicit exploration in our case).
For performance evaluation purposes, we ``slice'' different snapshots of the regret in \cref{fig:slices} at iteration counts $2\times 10^{5}$ and $2\times 10^{5}$.
In both cases, we observe a dramatic drop in variance for the Kernel algorithm relative to the Grid strategy, with a fixed number of arms uniformly cut beforehand;
we also note that the performance of the Kernel method approaches the theoretical slope of $\nRuns^{-1/3}$ that characterizes the Kernel method.

By contrast, the mean regret for the Grid approach seems to converge to a finite value which indicates a much slower regret minimization rate;
on the other hand, the mean regret of the Kernel method converges to $0$ at the anticipated rate.
% Code of the generated figures are made publicly available\footnote{TODO: add here Gitlab link (even temporarly empty)}.

%**********************************************************************
%***    BIBLIOGRAPHY
%**********************************************************************
\bibliographystyle{plainfull}
\bibliography{IEEEabrv,bibtex/Bibliography-PM}

\end{document}